\RequirePackage[table,xcdraw,svgnames]{xcolor}
\documentclass[lettersize,journal]{IEEEtran}

\usepackage[sort,nocompress]{cite}
\usepackage{amssymb}
\usepackage{bbm}
\usepackage{multirow}
\usepackage{booktabs}
\usepackage{adjustbox}
\newcommand{\myref}[1]{Eq.~\eqref{#1}}
\usepackage{svg}
\usepackage{enumitem}
\usepackage{microtype}
\usepackage{graphicx}
\usepackage{booktabs}
\usepackage{placeins}
\usepackage[normalem]{ulem}
\useunder{\uline}{\ul}{}
\usepackage{setspace}
\usepackage{hyperref}
\usepackage{url}

\definecolor{bestcell}{RGB}{211, 204, 241}
\definecolor{secondcell}{RGB}{227, 242, 253}

\usepackage{amsmath}
\usepackage{amssymb}
\usepackage{mathtools}
\usepackage{amsthm}

\usepackage[capitalize,noabbrev]{cleveref}

\crefname{figure}{Fig.}{Figs.}
\Crefname{figure}{Fig.}{Figs.}
\theoremstyle{plain}
\newtheorem{theorem}{Theorem}
\newtheorem{proposition}{Proposition}

\newtheorem{assumption}{Assumption}

\theoremstyle{remark}

\usepackage{amsmath,amsfonts}

\usepackage{algorithm}
\usepackage{algpseudocode}
\usepackage{array}
\usepackage[caption=false,font=normalsize,labelfont=sf,textfont=sf]{subfig}
\usepackage{textcomp}
\usepackage{stfloats}
\usepackage{url}
\usepackage{verbatim}
\usepackage{graphicx}
\def\BibTeX{{\rm B\kern-.05em{\sc i\kern-.025em b}\kern-.08em
    T\kern-.1667em\lower.7ex\hbox{E}\kern-.125emX}}
\usepackage{balance}
\begin{document}
\title{Latent Block-Diffusion Temporal Point Processes: A Semi-Autoregressive Framework for Asynchronous Event Sequence Generation}
\author{Shuai Zhang, Yancheng Chen, Chuan Zhou, 
\IEEEmembership{Member, IEEE}, Yang Liu, Xixun Lin, Xiangyu Zhao, \\Jun Zhu, \IEEEmembership{Fellow, IEEE}, and Zhi-Ming Ma
\thanks{Shuai Zhang, Yancheng Chen, Chuan Zhou, Yang Liu, and Zhi-Ming Ma are with the Academy of Mathematics and Systems Science, Chinese Academy of Sciences, Beijing 100190, China (e-mail: \protect\url{zhangshuai2021@amss.ac.cn}; \protect\url{chenyancheng22@mails.ucas.ac.cn};  \protect\url{zhouchuan@amss.ac.cn}; \protect\url{liuyang2020@amss.ac.cn}; \protect\url{mazm@amt.ac.cn}).}
\thanks{Xixun Lin is with the Institute of Information Engineering, Chinese Academy of Sciences, Beijing 100093, China (e-mail: \protect\url{linxixun@iie.ac.cn}).}
\thanks{Xiangyu Zhao is with the Department of Data Science, City University of Hong Kong, Hong Kong 999077, China (e-mail: \protect\url{xy.zhao@cityu.edu.hk}).}
\thanks{Jun Zhu is with the Department of Computer Science and Technology, Tsinghua University, Beijing 100084, China (e-mail: \protect\url{dcszj@tsinghua.edu.cn}).}

}

\markboth{Journal of \LaTeX\ Class Files,~Vol.~18, No.~9, September~2020}
{Latent Block-Diffusion Temporal Point Processes: A Semi-Autoregressive Framework for Asynchronous Event Sequence Generation}

\maketitle

\begin{abstract}
Modeling and sampling from the underlying distribution of asynchronous event sequences are crucial in various real-world applications, including social networks, medical diagnosis, and financial transactions. Existing autoregressive methods suffer from error accumulation during multi-step generation, while non-autoregressive diffusion methods are typically limited to fixed-length output sequences. In this paper, we propose Latent Block-Diffusion Temporal Point Processes (LBDTPP), a novel semi-autoregressive TPP framework that introduces a latent block diffusion mechanism for high-quality and variable-length event sequence generation. The core idea is to define an autoregressive probability distribution over event blocks in latent space and perform Gaussian diffusion within each block. By sequentially generating blocks while simultaneously sampling events in each block, LBDTPP preserves the length flexibility of autoregressive TPPs and inherits the parallel high-quality generation capability of diffusion models. Theoretically, we derive Wasserstein error bounds showing that, under suitable local approximation and prefix-stability assumptions, block-wise generation can reduce error accumulation compared with event-wise autoregressive generation. Extensive experiments on six real-world benchmark datasets demonstrate that LBDTPP outperforms state-of-the-art TPP baselines in both unconditional and conditional generation tasks. Further empirical analyses verify the benefits of latent-space diffusion and block-wise generation, and reveal the trade-off between generation quality and block size. Our code is available at \url{https://github.com/Zh-Shuai/LBDTPP}.
\end{abstract}
\section{Introduction}
\IEEEPARstart{A}{synchronous} event sequences are abundant in many real-world applications, including social networks \cite{fan2022hawkes,chen2024easydgl,gao2024learning}, medical diagnosis \cite{enguehard2020neural,liu2021event,shchur2021detecting}, and financial transactions \cite{shi2023language,zeng2024interacting,jiang2025danmakutppbench}. Each event in a sequence consists of a continuous timestamp and a discrete mark, representing when and what the event occurred. For example, in a social network, an event may record that a user interacted with a post at a specific time, with the mark indicating the corresponding interaction type such as posting, commenting, or sharing. Faithfully modeling and generating such sequences are essential for understanding complex temporal dynamics and supporting decision-making in various domains. Based on the availability of historical information, event sequence generation tasks can be broadly categorized into unconditional generation \cite{ogata1981lewis,ludke2023add,ludkeunlocking}, which aims to simulate high-fidelity event sequences from the underlying data distribution, and conditional generation \cite{du2016recurrent,zuo2020transformer,xue2023easytpp}, which focuses on predicting event occurrences given historical observations.

Temporal point processes (TPPs) \cite{daley2003introduction,snyder2012random,lacoste2005point,ortner2008marked} are the dominant modeling framework for asynchronous event sequences. Most existing TPP models, including Poisson processes \cite{kingman1992poisson}, self-correcting processes \cite{isham1979self}, Hawkes processes \cite{hawkes1971spectra} and their Transformer variants \cite{zhang2020self,zuo2020transformer,yang2022transformer}, follow an autoregressive paradigm, where events are modeled sequentially, with each event conditioned on the preceding history. While autoregressive TPPs naturally support variable-length generation and have achieved strong performance in predicting the next event, their one-by-one generation procedure suffers from error accumulation during multi-step generation \cite{ludke2023add,kerrigan2024eventflow}. Small errors introduced at early steps may propagate through subsequent steps and gradually amplify, leading to substantial degradation in generation quality \cite{xue2022hypro,zeng2024interacting}.

Recently, diffusion probabilistic models \cite{sohl2015deep,ho2020denoising,songdenoising} have emerged as a powerful framework for generative modeling, with notable applications in computer vision \cite{songscore,rombach2022high,lu2022dpm,bao2023all} and natural language processing \cite{li2022diffusion,nie2025large,wang2026diffusion}. Building upon this framework, non-autoregressive diffusion TPPs have been proposed for modeling event sequences \cite{ludke2023add,zeng2024interacting,kerrigan2024eventflow}. By generating multiple events simultaneously, these models avoid the one-by-one sampling of autoregressive TPPs and achieve improved performance in multi-step forecasting tasks. Nevertheless, existing diffusion-based approaches for event sequences typically model the conditional distribution of a fixed-length future sequence given historical events \cite{zeng2024interacting, zhou2025non}. When applied to unconditional generation within a time interval $[0,T]$, such non-autoregressive methods need to model the joint distribution of the entire sequence and generate it in a single shot. As a result, this paradigm is limited to producing sequences with a pre-specified number of events, reducing flexibility and making it unsuitable for realistic scenarios where sequence lengths are unknown and variable.

To mitigate the issues of error accumulation in autoregressive TPPs and fixed-length generation in non-autoregressive diffusion TPPs, we introduce Latent Block-Diffusion Temporal Point Processes (LBDTPP), a novel semi-autoregressive TPP framework that supports high-quality, variable-length event sequence generation in unconditional and conditional settings. LBDTPP decomposes event sequence generation into two levels: (i) sequential generation across blocks to preserve event dependencies and support variable-length generation, and (ii) parallel generation of multiple events within each block via Gaussian diffusion to alleviate error accumulation. This design enables our model to combine the strengths of prior paradigms: it retains the length flexibility of autoregressive TPPs, while obtaining the parallel high-quality generation capability of non-autoregressive diffusion TPPs.

Specifically, LBDTPP draws inspiration from discrete block diffusion models \cite{arriolablock} developed for token generation in natural language, but introduces a latent block diffusion formulation for asynchronous event sequences. Unlike discrete text tokens, event data couple continuous timestamps with discrete marks, making both discrete and continuous diffusion unsuitable for direct application. To this end, LBDTPP first maps each event into a continuous latent space and then factorizes the latent sequence distribution as a product of conditional distributions over event blocks, while performing Gaussian diffusion within each block. Our model sequentially samples latent blocks, generates multiple event representations in parallel within each block, and decodes them back to the original event space, enabling variable-length and high-quality event sequence generation. We further derive Wasserstein generation-error bounds under local approximation and prefix-stability assumptions, indicating that block-wise sampling can reduce prefix-level error accumulation by shortening the recursive sampling horizon from events to blocks. Extensive experiments on six real-world datasets demonstrate that LBDTPP outperforms state-of-the-art TPP baselines, and subsequent analysis shows that the gains stem from latent-space diffusion and block-wise generation.

Our main contributions are as follows:

\begin{itemize}[leftmargin=*]
    \item We introduce LBDTPP, a latent block diffusion framework for modeling asynchronous event sequences. By factorizing the latent sequence distribution across event blocks and performing Gaussian diffusion within each block, LBDTPP forms a semi-autoregressive generation paradigm that supports variable-length generation and parallel high-quality sampling, while mitigating error accumulation in autoregressive TPPs and overcoming the fixed-length limitation of non-autoregressive diffusion TPPs.
    
    \item We provide a theoretical analysis of error accumulation for event-wise and block-wise generation. Under explicit local approximation and prefix-stability assumptions, we derive Wasserstein bounds showing that event-wise autoregressive generation accumulates errors over all event-level sampling steps, whereas block-wise generation accumulates errors only over block-level transitions. This analysis explains why block-wise generation can mitigate prefix-level error accumulation and clarifies the block size trade-off observed empirically.

    \item We conduct extensive experiments on six real-world benchmark datasets across multiple domains. Experimental results demonstrate that LBDTPP outperforms both autoregressive and non-autoregressive TPP baselines in unconditional and conditional generation tasks. Further analysis validates the benefits of latent-space diffusion and block-wise generation. Sampling time comparisons show that LBDTPP achieves competitive generation efficiency, and its fast version can be faster than all baseline models.
\end{itemize}

\section{Related Work}\label{sec:related work}

In this section, we review TPP-based methods for event sequence modeling and generation, including autoregressive TPPs and non-autoregressive diffusion TPPs. We also briefly discuss block diffusion models, which provide methodological inspiration for our latent block diffusion TPPs.

\subsection{Autoregressive TPPs} 

Temporal point processes (TPPs) \cite{daley2003introduction,favreau2019extracting,lin2024extensive,zhou2026advances} are a class of stochastic processes for modeling sequences of random events in continuous time. Most existing TPP models follow an autoregressive paradigm, where each event is modeled conditioned on its preceding events. Early works \cite{hawkes1971spectra,isham1979self,snyder2012random,farajtabar2017coevolve} rely on parametric conditional intensity functions (CIFs) to characterize the expected rate of event occurrences given the history. More recent approaches employ neural networks to learn more flexible CIFs \cite{xiao2017modeling,mei2017neural,zuo2020transformer,zhang2024neural} or conditional probability density functions \cite{shchurintensity,zhang2023multiple,panos2024decomposable,zhang2025conformal}. Although autoregressive TPPs naturally support variable-length generation and perform well in next-event prediction, their sampling procedure remains inherently sequential, generating events one by one. Consequently, errors from early steps may propagate through subsequent events and accumulate during long-horizon generation \cite{xue2022hypro,ludke2023add,zeng2024interacting}. In contrast, our LBDTPP model generates multiple high-quality events simultaneously within each block, alleviating the error accumulation issue and having the potential to improve sampling efficiency.

\subsection{Non-autoregressive Diffusion TPPs} 

Diffusion-based TPP models have emerged as a promising approach for event sequence modeling. DSTPP \cite{yuan2023spatio} adopts denoising diffusion models to capture spatio-temporal event dynamics. AddThin \cite{ludke2023add} and PSDiff \cite{ludkeunlocking} leverage the thinning and superposition properties of point processes \cite{daley2003introduction} to design diffusion-like models on the positive real space and general metric spaces, respectively. However, none of these methods has been explored for TPPs with discrete marks. EventFlow \cite{kerrigan2024eventflow} and EdiTPP \cite{ludke2026edit} employ flow matching \cite{lipmanflow,havasiedit} for unconditional and conditional generation of event timestamp sequences, without modeling event marks. CDiff \cite{zeng2024interacting} introduces two interacting diffusion processes for long-horizon marked event forecasting, but as a non-autoregressive diffusion model, it models fixed-length future sequences given the history. When directly applied to unconditional generation, such an approach needs to model the complete sequence and generate it in one shot, and thus can only produce sequences with a pre-specified number of events. In contrast, LBDTPP generates marked event sequences block by block, enabling variable-length generation in both unconditional and conditional settings.

\subsection{Block Diffusion Models} 

Block diffusion models \cite{han2023ssd,arriolablock,arriola2025encoder} have been proposed to integrate the strengths of autoregressive and diffusion language models, supporting variable-length, high-quality generation and improving inference efficiency with key-value (KV) caching and parallel sampling. These methods are mainly designed for discrete token sequences: they partition tokens into blocks and model each block conditioned on preceding blocks with discrete diffusion \cite{hoogeboom2021argmax,austin2021structured}, enabling parallel generation within blocks while maintaining dependencies across blocks. Such a semi-autoregressive paradigm has shown promise in video generation~\cite{ren2025next,chen2025sana,zhang2025blockvid} and diffusion large language models \cite{nie2025large,wang2026diffusion,wu2025fast,wu2026fast}. Different from these works, LBDTPP targets asynchronous event sequences, where each event contains a continuous timestamp and a discrete mark. We therefore introduce latent block diffusion, which embeds mixed-type events into a latent space and performs continuous Gaussian diffusion within latent blocks. This latent formulation preserves the length flexibility and intra-block parallelism of block diffusion, while making it suitable for event sequence modeling and generation.

\section{Preliminary}\label{sec:preliminary}

In this section, we provide an overview of temporal point processes and diffusion probabilistic models. Throughout this paper, the symbol ``$t$'' denotes the event timestamp, and ``$k$'' denotes the $k$-th step in the diffusion process.

\subsection{Temporal Point Processes}
Given a set of asynchronous event sequences drawn from the data distribution $q(\mathbf{x})$, where each sequence is represented as $\mathbf{x}=\left(\mathbf{x}^1, \ldots, \mathbf{x}^L\right)$, and the $\ell$-th event $\mathbf{x}^{\ell}=(t^{\ell},m^{\ell})$ consists of the occurrence timestamp $t^{\ell}\in\mathbb{R}_{+}$ and the mark $m^{\ell}\in[M]:=\{1,\ldots,M\}$, with $t^{\ell-1}<t^{\ell}$. The sequence length $L$, i.e., the number of events in the sequence, can vary across different sequences. It is worth mentioning that the event timestamps can be equivalently expressed as the inter-event times $\tau^{\ell}=t^{\ell}-t^{\ell-1}\in\mathbb{R}_{+}$, where $t^0=0$. If not otherwise specified, we use the inter-event time representation, i.e., $\mathbf{x}^{\ell}=(\tau^{\ell},m^{\ell})$. The goal is to fit a model $p_\theta(\mathbf{x})$ of $q(\mathbf{x})$ to learn the event sequence distribution, and then use the learned model to generate high-fidelity sequences or make accurate event predictions.

The common approach for modeling event sequences is using temporal point processes (TPPs) \cite{daley2003introduction}, which characterize the occurrence of discrete events in continuous time by defining conditional intensity functions (CIFs) \cite{du2016recurrent} or conditional probability density functions (PDFs) \cite{shchurintensity}. Most existing TPP models, such as Poisson processes \cite{kingman1992poisson}, self-correcting processes \cite{isham1979self}, Hawkes processes \cite{hawkes1971spectra} and their Transformer variants \cite{zhang2020self,zuo2020transformer}, typically parameterize the distribution of $L$ events in the autoregressive (AR) form:
\begin{equation}
\log p_\theta(\mathbf{x})=\sum_{\ell=1}^L \log p_\theta\left(\mathbf{x}^{\ell} \mid \mathbf{x}^{<\ell}\right),
\end{equation}
where $\mathbf{x}^{<\ell}=\left(\mathbf{x}^1, \ldots, \mathbf{x}^{\ell-1}\right)$ denotes the historical events before the $\ell$-th event. Since the conditional PDF $p_\theta\left(\mathbf{x}^{\ell} \mid \mathbf{x}^{<\ell}\right)$ and the CIF $\lambda_\theta\left(\mathbf{x}^{\ell} \mid \mathbf{x}^{<\ell}\right)$ can be expressed in terms of each other~\cite{rasmussen2018lecture}, the above AR factorized distribution can also be equivalently specified using CIFs.

We emphasize that the autoregressive factorization used here is defined for a finite realized event sequence of length $L$, where the modeling target is the probability of the realized sequence itself \cite{boyd2023probabilistic}, rather than the full point-process likelihood up to a fixed terminal time $T$. Under the latter formulation, an additional survival term over the terminal interval would indeed appear. However, under our current sequence-level modeling setup, this term is not required, and omitting it does not affect subsequent model training or optimization objectives.

While autoregressive TPPs have been successful in generating a single subsequent event \cite{zuo2020transformer,yang2022transformer,xue2023easytpp}, their one-by-one sequential sampling procedure can lead to error accumulation in multi-step generation, thereby degrading the overall generation performance \cite{xue2022hypro,zeng2024interacting}.

\subsection{Diffusion Probabilistic Models}\label{sec:diffusion models}
To simplify the expression, we allow some symbolic abuse. In this subsection, we rewrite $\mathbf{x}\sim q(\mathbf{x})$ as a vector in the Euclidean space $\mathbb{R}^L$. Diffusion probabilistic models (hereafter diffusion models) \cite{sohl2015deep,ho2020denoising} overcome the aforementioned sequential sampling limitation by learning the distribution $p_\theta(\mathbf{x})$ directly, admitting parallel generation. 

Diffusion models define a forward process that gradually adds Gaussian noise to the clean data $\mathbf{x}_0=\mathbf{x}$:
\begin{align}
q\left(\mathbf{x}_{1: K} \mid \mathbf{x}_0\right)&=\prod_{k=1}^{K} q\left(\mathbf{x}_k \mid \mathbf{x}_{k-1}\right),\\
q\left(\mathbf{x}_k \mid \mathbf{x}_{k-1}\right)&=\mathcal{N}\left(\mathbf{x}_k ; \sqrt{\alpha_k} \mathbf{x}_{k-1}, (1-\alpha_k) \mathbf{I}\right),
\end{align}
where $\alpha_1,\ldots,\alpha_K$ are the decreasing values in $[0,1]$, and $K$ is the total number of diffusion steps.

On the other hand, the reverse denoising process starts from $p\left(\mathbf{x}_K\right)=\mathcal{N}\left(\mathbf{x}_K ; \mathbf{0}, \mathbf{I}\right)$ and proceeds as follows:
\begin{align}
p_\theta\left(\mathbf{x}_{0: K}\right)&=p\left(\mathbf{x}_K\right) \prod_{k=1}^K p_\theta\left(\mathbf{x}_{k-1} \mid \mathbf{x}_k\right), \\
p_{\mathbf{\theta}}\left(\mathbf{x}_{k-1} \mid \mathbf{x}_{k}\right)&=\mathcal{N}\left(\mathbf{x}_{k-1}; \boldsymbol{\mu}_{\mathbf{\theta}}\left(\mathbf{x}_{k}, k\right), \sigma_k^2\mathbf{I}\right),
\end{align}
where $\sigma_k^2=1-\alpha_k$. To learn the parameters $\theta$, the standard variational inference method involves minimizing the negative evidence lower bound (NELBO) \cite{luo2022understanding}:
\begin{equation}
-\log p_{\mathbf{\theta}}\left(\mathbf{x}_0\right) \leq \mathbb{E}_{q\left(\mathbf{x}_{1: K} \mid \mathbf{x}_0\right)}\left[-\log \frac{p_{\mathbf{\theta}}\left(\mathbf{x}_{0: K}\right)}{q\left(\mathbf{x}_{1: K} \mid \mathbf{x}_0\right)}\right].
\end{equation}
Prior work \cite{ho2020denoising} simplified this NELBO, replacing it with the following loss function:
\begin{equation}
\mathcal{L}_{\text {simple}}(\mathbf{x}_0;\theta)=\mathbb{E}_{k, \mathbf{x}_0, \boldsymbol{\epsilon}}\left[\left\|\boldsymbol{\epsilon}-\boldsymbol{\epsilon}_\theta\left(\mathbf{x}_k,k\right)\right\|^2\right],
\end{equation}
where $\bar{\alpha}_k=\prod_{s=1}^k \alpha_s$, $\boldsymbol{\epsilon} \sim \mathcal{N}(\mathbf{0}, \mathbf{I})$, and $\boldsymbol{\epsilon}_\theta$ is parameterized by neural networks to predict noise $\boldsymbol{\epsilon}$ using the noisy input $\mathbf{x}_k=\sqrt{\bar{\alpha}_k} \mathbf{x}_0+\sqrt{1-\bar{\alpha}_k} \boldsymbol{\epsilon}$ and the diffusion step $k$.

After training, a new data point, still denoted as $\mathbf{x}_0$ for convenience, that follows the learned distribution $p_\theta(\mathbf{x}_0)$, can be generated as follows. We start by sampling $\mathbf{x}_K \sim \mathcal{N}(\mathbf{0}, \mathbf{I})$, and then conduct the reverse process $p_{\mathbf{\theta}}\left(\mathbf{x}_{k-1} \mid \mathbf{x}_{k}\right)$ to iteratively sample $\mathbf{z} \sim \mathcal{N}(\mathbf{0}, \mathbf{I})$ and compute $\mathbf{x}_{k-1}$ for $k=K$ to $1$:
\begin{equation}
\mathbf{x}_{k-1}=\frac{1}{\sqrt{\alpha_{k}}}\big(\mathbf{x}_{k}-\frac{1-\alpha_{k}}{\sqrt{1-\bar{\alpha}_{k}}}\boldsymbol{\epsilon}_{\mathbf{\theta}}\left(\mathbf{x}_{k}, k\right)\big)+\sigma_k \mathbf{z}.
\end{equation}

Although the above standard diffusion model can generate all dimensions of $\mathbf{x}_0$ simultaneously, it is tied to a fixed-dimensional vector of length $L$. This fixed-length constraint limits its ability to simulate event sequences in real-world applications, where the number of events is unknown, varies across sequences, and may exceed the pre-specified length $L$ as time progresses. Moreover, if only a short sequence is needed at inference time, generating a full length-$L$ sequence introduces unnecessary computation and sampling time.

\begin{figure*}[t]
    \centering
    \includegraphics[width=1\linewidth]{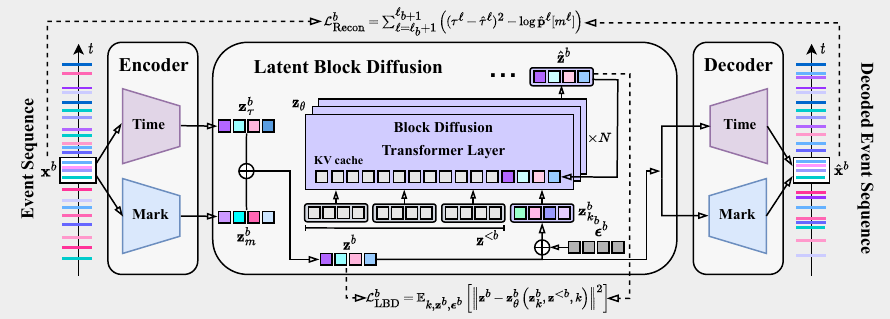}
    \vspace{-7pt}
    \caption{The end-to-end training framework of LBDTPP. The position of each event on the timeline represents its timestamp, and the color denotes its event mark. Our model first embeds the event sequence into a continuous latent space, then learns the latent distribution by autoregressively modeling event blocks and performing Gaussian diffusion within each block, while reconstructing the input event sequence from the latent representations. We jointly train the block diffusion Transformer and the encoder-decoder by minimizing a weighted combination of the latent block diffusion loss and the reconstruction loss.}
    \label{fig:LBDTPP framework}
\end{figure*}
\section{Proposed Method: LBDTPP}\label{sec:method}
In this section, we present Latent Block-Diffusion Temporal Point Processes (LBDTPP), a novel semi-autoregressive TPP framework for modeling asynchronous event sequences. LBDTPP first maps events into a continuous latent space and then models the latent sequence representation in a block-wise manner: it factorizes dependencies autoregressively across event blocks, and performs Gaussian diffusion to learn each block distribution. At inference time, LBDTPP sequentially generates latent event blocks, produces multiple latent event representations in parallel within each block, and subsequently decodes them back into the event space. This design preserves the variable-length generation ability of autoregressive TPPs while inheriting the parallel high-quality generation capability of non-autoregressive diffusion TPPs, and meanwhile mitigates their respective limitations of error accumulation and fixed-length generation. The training framework of LBDTPP is illustrated in \Cref{fig:LBDTPP framework}.

In what follows, we first introduce the model architecture in \Cref{sec:model architecture}, which consists of the event encoder, latent block diffusion, and event decoder. Then, we detail the training and sampling procedures for unconditional event sequence generation in Sections~\ref{sec:training} and \ref{sec:Unconditional Generation}, respectively. We next describe the extension to conditional generation in \Cref{sec:Conditional Generation}, and then provide a theoretical analysis of generation-error accumulation in \Cref{sec:theoretical analysis}.

\subsection{Model Architecture}\label{sec:model architecture}

Since asynchronous event sequences contain both continuous timestamps and discrete marks, neither standard continuous diffusion \cite{ho2020denoising} nor discrete diffusion models \cite{austin2021structured} can be directly applied to raw event sequences. We therefore first construct continuous latent event representations that jointly encode temporal and mark information, providing a homogeneous space for subsequent generative modeling. Below, we describe the architecture components in detail.

\textbf{Event Encoder.} For each event $\mathbf{x}^{\ell}=(\tau^{\ell}, m^{\ell})$ in the sequence $\mathbf{x}=(\mathbf{x}^{1},\ldots,\mathbf{x}^{L})$, we first encode its temporal and mark information separately. Specifically, the inter-event time is mapped to a time embedding $\mathbf{z}_{\tau}^{\ell}=\operatorname{TimeEmbed}(\tau^{\ell})\in\mathbb{R}^{D}$ using positional encoding \cite{zuo2020transformer}, while the mark embedding $\mathbf{z}_m^{\ell}=\operatorname{MarkEmbed}(m^{\ell})\in\mathbb{R}^{D}$ is obtained by applying a linear transformation to the one-hot representation of $m^{\ell}$:
\begin{align}
\left[\operatorname{TimeEmbed}(\tau^{\ell})\right]_d &= \!\left\{\begin{array}{ll}
\!\!\!\cos \big(\tau^{\ell} / 10000^{\frac{d-1}{D}}\big), & \!\!\!\text {if } d \text { is odd},\vspace{2pt}\\
\!\!\!\sin \big(\tau^{\ell} / 10000^{\frac{d}{D}}\big), & \!\!\!\text {if } d \text { is even},
\end{array}\right. \nonumber\\
\operatorname{MarkEmbed}(m^{\ell})&=\mathbf{W}\cdot\operatorname{OneHot}(m^{\ell}), \label{eq:mark-embed}
\end{align}
where $d=0,\ldots,D-1$, and $\operatorname{OneHot}(\cdot):[M]\rightarrow\{0,1\}^M$ denotes the one-hot encoding function. The embedding matrix $\mathbf{W}\in\mathbb{R}^{D\times M}$ is initialized from a uniform distribution and kept fixed during training. Our experiments show that keeping it fixed yields slightly better empirical performance. As a result, the event encoder contains no learnable parameters.

The overall representation of event $\mathbf{x}^{\ell}$ is then obtained by adding its time and mark embeddings, i.e., $\mathbf{z}^{\ell}=\mathbf{z}_{\tau}^{\ell}+\mathbf{z}_m^{\ell}\in\mathbb{R}^{D}$. In this work, we use addition as the default fusion operation, while other operations such as concatenation are also compatible with our framework. By stacking the event representations in temporal order, we form the representation of the entire sequence $\mathbf{x}$ of length $L$ as $\mathbf{z}=(\mathbf{z}^{1},\ldots,\mathbf{z}^{L})\in\mathbb{R}^{L\times D}$.

Given the sequence representation $\mathbf{z}$, we aim to model its distribution while mitigating error accumulation caused by one-by-one autoregressive sampling and overcoming fixed-length non-autoregressive generation. To this end, we propose to model $\mathbf{z}$ at the block level, which captures event dependencies across blocks, supports variable-length generation, and enables parallel diffusion-based generation within each block.

\textbf{Latent Block Diffusion.} We partition the latent sequence representation $\mathbf{z}=\left(\mathbf{z}^1, \ldots, \mathbf{z}^L\right)$ into $B:= L/L^{\prime}$ non-overlapping blocks of length $L^{\prime}$, and assume $B$ is an integer (if not, we pad the raw event sequence in advance so that $L$ is divisible by $L^{\prime}$). For each $b\in[B]$, we denote the $b$-th event block $\left(\mathbf{z}^{\ell_b+1}, \ldots, \mathbf{z}^{\ell_{b+1}}\right)$ simply as $\mathbf{z}^{b}$, where $\ell_b=(b-1)L^{\prime}$. Thus, $\mathbf{z}^{b}$ contains $L^{\prime}$ consecutive event representations. We denote the historical blocks before block $b$ as $\mathbf{z}^{<b}=\left(\mathbf{z}^{1}, \ldots, \mathbf{z}^{\ell_b}\right)$. 

Here and below, superscript $\ell$ indexes individual events, so $\mathbf{z}^{\ell}\in\mathbb{R}^{D}$, whereas superscript $b$ indexes blocks, so $\mathbf{z}^{b}\in\mathbb{R}^{L^{\prime}\times D}$. This distinction avoids ambiguity between event-level and block-level representations.

Different from discrete block diffusion \cite{arriolablock} designed for token states, our latent block diffusion operates on continuous event representations that jointly encode timestamp and mark information. Specifically, we model the distribution of the latent sequence representation by factorizing it autoregressively over blocks and performing Gaussian diffusion within each block. That is, we decompose the log-likelihood of the sequence representation $\mathbf{z}$ over blocks as:
\begin{equation}\label{equ:block factorization}
\log p_{\mathbf{\theta}}(\mathbf{z})=\sum_{b=1}^B \log p_{\mathbf{\theta}}\left(\mathbf{z}^{b} \mid \mathbf{z}^{<b}\right),
\end{equation}
where each conditional distribution $p_{\mathbf{\theta}}\left(\mathbf{z}^{b} \mid \mathbf{z}^{<b}\right)$ is learned by conducting Gaussian diffusion over a block of $L^{\prime}$ event representations. The sequential factorization captures dependencies across blocks and supports variable-length sequence generation through block-by-block generation, while the within-block diffusion learns the conditional distribution of each block and enables parallel high-quality generation of multiple events to reduce error accumulation. The detailed generation procedure will be described in \Cref{sec:Unconditional Generation}, while this part focuses on the modeling formulation.

For each block $b\in[B]$, to learn the conditional distribution $p_{\mathbf{\theta}}\left(\mathbf{z}^{b} \mid \mathbf{z}^{<b}\right)$, we define a forward diffusion process that gradually adds Gaussian noise to the clean block $\mathbf{z}_0^{b}=\mathbf{z}^{b}$:
\begin{equation}\label{eq:q_joint}
q\left(\mathbf{z}_{1: K}^b \mid \mathbf{z}_0^b\right)
= \prod_{k=1}^{K} q\left(\mathbf{z}_k^b \mid \mathbf{z}_{k-1}^b\right),
\end{equation}
\begin{equation}\label{equ:forward definition}
q\left(\mathbf{z}_k^b \mid \mathbf{z}_{k-1}^b\right)
= \mathcal{N}\left(\mathbf{z}_k^b ; \sqrt{\alpha_k} \mathbf{z}_{k-1}^b, (1-\alpha_k) \mathbf{I}\right),
\end{equation}
where $\mathbf{z}_k^{b}$ denotes the noisy $b$-th block at diffusion step $k$. This noise-adding process operates solely on the currently considered block $b$ and is independent of all other blocks, including both historical and subsequent blocks. Thus, Eqs.~\eqref{eq:q_joint} and \eqref{equ:forward definition} are not conditioned on $\mathbf{z}^{<b}$. Similar to standard diffusion models \cite{ho2020denoising}, we can sample $\mathbf{z}_k^{b}$ directly based on $\mathbf{z}_0^{b}$:
\begin{equation}\label{equ:closed distribution form}
q(\mathbf{z}_k^{b} \mid \mathbf{z}_0^{b})=\mathcal{N}\left(\mathbf{z}_k^{b} ; \sqrt{\bar{\alpha}_k} \mathbf{z}_0^{b}, (1-\bar{\alpha}_k) \mathbf{I}\right),
\end{equation}
where $\bar{\alpha}_k=\prod_{s=1}^k \alpha_s$. In other words, the noisy block at any step $k$ has the closed-form expression: 
\begin{equation}\label{equ:closed form}
\mathbf{z}_k^{b}=\sqrt{\bar{\alpha}_k} \mathbf{z}_0^{b}+\sqrt{1-\bar{\alpha}_k} \boldsymbol{\epsilon}^b, \quad\text{where} \,\,\boldsymbol{\epsilon}^b \sim \mathcal{N}(\mathbf{0}, \mathbf{I}).
\end{equation}
Here the added noise $\boldsymbol{\epsilon}^b$ is independent across different blocks.

The corresponding reverse denoising process, defined on block $b$, starts from $p\left(\mathbf{z}_K^{b} \mid \mathbf{z}^{<b}\right)=\mathcal{N}\left(\mathbf{z}_K^{b} ; \mathbf{0}, \mathbf{I}\right)$ and proceeds as follows:
\begin{equation}\label{eq:p_joint}
\!\!\!\!p_{\mathbf{\theta}}\left(\mathbf{z}_{0: K}^b\mid \mathbf{z}^{<b}\right)=p\left(\mathbf{z}_K^b \mid \mathbf{z}^{<b}\right) \prod_{k=1}^K p_{\mathbf{\theta}}\left(\mathbf{z}_{k-1}^b \mid \mathbf{z}_k^b, \mathbf{z}^{<b}\right),
\end{equation}
\begin{equation}\label{equ:reverse definition}
p_{\mathbf{\theta}}\left(\mathbf{z}_{k-1}^{b} \mid \mathbf{z}_k^{b}, \mathbf{z}^{<b}\right)=\mathcal{N}\left(\mathbf{z}_{k-1}^{b}; \boldsymbol{\mu}_{\mathbf{\theta}}^b\left(\mathbf{z}_k^{b}, \mathbf{z}^{<b}, k\right), \sigma_k^2\mathbf{I}\right).
\end{equation}
Notably, the reverse process at block $b$ is conditioned on all preceding blocks $\mathbf{z}^{<b}$, enabling the model to learn the conditional distribution $p_{\mathbf{\theta}}\left(\mathbf{z}^{b} \mid \mathbf{z}^{<b}\right)$ and capture dependencies from historical event blocks.

Based on the forward and reverse processes described above, we can derive the negative evidence lower bound (NELBO) for each term $\log p_{\mathbf{\theta}}\left(\mathbf{z}^{b} \mid \mathbf{z}^{<b}\right)$ in \myref{equ:block factorization}, and then simplify it to the following loss for block $b$:
\begin{equation}\label{equ:element block diffusion loss}
\mathcal{L}_{\text {LBD}}^{b}(\mathbf{z}^{b}, \mathbf{z}^{<b}; \mathbf{\theta})=\mathbb{E}_{k, \mathbf{z}^b, \boldsymbol{\epsilon}^b}\left[\left\|\mathbf{z}^b-\mathbf{z}_{\mathbf{\theta}}^b\left(\mathbf{z}_k^b, \mathbf{z}^{<b}, k\right)\right\|^2\right],
\end{equation}
where $k\sim\operatorname{Unif}\left(\{1,\ldots,K\}\right)$, and the noisy block $\mathbf{z}_k^b$ is obtained according to \myref{equ:closed form}. Moreover, the relationship between the block predictor $\mathbf{z}_{\mathbf{\theta}}^b$ and the reverse process mean $\boldsymbol{\mu}_{\mathbf{\theta}}^b$ is given by:
\begin{equation}\label{equ:reverse mean parameterization}
\begin{aligned}
\boldsymbol{\mu}_{\mathbf{\theta}}^b\left(\mathbf{z}_k^b,\mathbf{z}^{<b},k\right)
= {} & \frac{\sqrt{\alpha_k}(1-\bar{\alpha}_{k-1})\mathbf{z}_k^b}{1-\bar{\alpha}_k} \\
& + \frac{\sqrt{\bar{\alpha}_{k-1}}(1-\alpha_k)}{1-\bar{\alpha}_k}
\mathbf{z}_{\mathbf{\theta}}^b\left(\mathbf{z}_k^b,\mathbf{z}^{<b},k\right).
\end{aligned}
\end{equation}

In \myref{equ:element block diffusion loss}, we use the denoising model $\mathbf{z}_{\mathbf{\theta}}^b\left(\mathbf{z}_k^{b}, \mathbf{z}^{<b}, k\right)$ to directly predict the clean latent block $\mathbf{z}^b$ (with the predicted output denoted as $\hat{\mathbf{z}}^b$) using the noisy block $\mathbf{z}_k^{b}$, historical blocks $\mathbf{z}^{<b}$, and diffusion step $k$. From our early experiments, we found that this $\mathbf{z}^b$-prediction approach outperforms predicting the block noise $\boldsymbol{\epsilon}^b$. The $B$ denoisers $\mathbf{z}_{\mathbf{\theta}}^b$ are parameterized by a single Transformer \cite{vaswani2017attention}, $\mathbf{z}_{\mathbf{\theta}}$, which is equipped with a specialized attention mask to facilitate efficient training as described in \cref{sec:training}.

We then aggregate the latent block diffusion loss over all event blocks as:
\begin{equation}\label{equ:total block diffusion loss}
\mathcal{L}_{\text {LBD}}(\mathbf{z}; \mathbf{\theta})=\frac{1}{L}\sum_{b=1}^B \mathcal{L}_{\text {LBD}}^{b}(\mathbf{z}^{b}, \mathbf{z}^{<b}; \mathbf{\theta}).
\end{equation}

Formally, we summarize the NELBO of our LBDTPP model and its simplification to the above loss function in the following proposition.
\begin{proposition}[NELBO of LBDTPP]
\label{prop:blockwise latent elbo}
Under the sequential factorization over event blocks in \myref{equ:block factorization}, suppose the block-wise forward diffusion process satisfies Eqs.~\eqref{eq:q_joint} and \eqref{equ:forward definition}, and the reverse denoising process for each block $b\in[B]$ is given by Eqs.~\eqref{eq:p_joint} and \eqref{equ:reverse definition}. Then the negative log-likelihood of the latent event sequence representation satisfies

\begin{equation}\label{equ:blockwise latent elbo}
-\log p_{\mathbf{\theta}}(\mathbf{z})
\leq
\sum_{b=1}^B \mathcal{J}_b(\mathbf{z}^b,\mathbf{z}^{<b};\theta),
\end{equation}
where $\mathcal{J}_b(\mathbf{z}^b,\mathbf{z}^{<b};\theta)$ is the standard NELBO for the conditional distribution $p_{\mathbf{\theta}}(\mathbf{z}^b\mid\mathbf{z}^{<b})$. Moreover, similar to \cite{ho2020denoising}, the upper bound in \myref{equ:blockwise latent elbo} can be further simplified to the surrogate latent block diffusion loss in \myref{equ:total block diffusion loss}.
\end{proposition}

\vspace{0.2cm}
The proof is provided in the supplementary material. The resulting latent block diffusion loss $\mathcal{L}_{\text {LBD}}(\mathbf{z}; \mathbf{\theta})$ serves as the latent distribution learning objective: it trains the denoising Transformer to recover clean latent event blocks from noisy ones conditioned on historical blocks, enabling the reverse process to sample coherent latent blocks at inference time. To obtain actual event sequences, the generated latent representations need to be mapped back to the event space.

\begin{algorithm}[t]
\caption{LBDTPP Training}
\label{alg:unconditional training}
\textbf{Input:} event sequence $\mathbf{x}$ of length $L$, block size $L^{\prime}$, diffusion steps $K$

\textbf{repeat}

\begin{enumerate}[label=\arabic*.]
    \item $\mathbf{z}$ $\leftarrow$ $\texttt{Encoder}(\mathbf{x})$; \,$\hat{\mathbf{x}}$ $\leftarrow$ $\texttt{Decoder}(\mathbf{z})$\vspace{4pt} 
    \item Sample $k_1, \dots, k_B \sim \operatorname{Unif}\left(\{1, \dots, K\}\right)\hfill \triangleright \text{\small $B=L/L^{\prime}$}$\vspace{4pt} 
    \item $\forall b \in \{1, \dots, B\}$: \, $\mathbf{z}_{k_b}^{b} \sim q(\,\cdot\mid\mathbf{z}^b) \hfill \triangleright$ $\text{\small \myref{equ:closed distribution form}}$\vspace{4pt} 
    \item $\emptyset, \mathbf{K}^{1: B}, \mathbf{V}^{1: B} \leftarrow \mathbf{z}_{\mathbf{\theta}}(\mathbf{z}) \hfill \triangleright$ $\text{\small $\mathrm{KV}$ cache}$ \vspace{4pt} 
    \item $\forall b$: \,$\hat{\mathbf{z}}^b, \emptyset, \emptyset \leftarrow \mathbf{z}_{\mathbf{\theta}}^b\left(\mathbf{z}_{k_b}^{b}, \mathbf{K}^{1:b-1}, \mathbf{V}^{1:b-1}, k_b\right)$ \vspace{4pt} 
    \item $\hat{\mathbf{z}} \leftarrow \hat{\mathbf{z}}^{1} \oplus \cdots \oplus \hat{\mathbf{z}}^{B}$ \vspace{4pt} 
    \item Take gradient descent step on $\nabla_{\mathbf{\theta}, \mathbf{\phi}}\, \mathcal{L}_{\text {Overall}}(\mathbf{x};\mathbf{\theta},\mathbf{\phi})$ \vspace{2pt}
\end{enumerate}
\vspace{0.5pt}
\textbf{until} converged
\end{algorithm}

\vspace{0.3cm}
\textbf{Event Decoder.} We design an event decoder to reconstruct the event sequence $\mathbf{x}$ from its latent representation $\mathbf{z}$, so that latent blocks sampled during inference can be decoded back to the original event space. Specifically, the $\ell$-th reconstructed event $\hat{\mathbf{x}}^{\ell}=(\hat{\tau}^{\ell}, \hat{m}^{\ell})$ is decoded from $\mathbf{z}^{\ell}$ by:

\begin{equation}\label{eq:time reconstruction}
\hat{\tau}^{\ell}=\operatorname{Softplus}\left(\operatorname{MLP}_{\tau}\left(\mathbf{z}^{\ell}\right)\right),
\end{equation}
\begin{equation}\label{eq:mark reconstruction}
  \hat{m}^{\ell}=\underset{m}{\operatorname{argmax}}\,\hat{\mathbf{p}}^{\ell}[m],  \quad\hat{\mathbf{p}}^{\ell}= \operatorname{Softmax}\left(\operatorname{MLP}_{m}\left(\mathbf{z}^{\ell}\right)\right),
\end{equation}
where $\operatorname{MLP}_{\tau}(\cdot): \mathbb{R}^{D}\rightarrow\mathbb{R}$ and $\operatorname{MLP}_{m}(\cdot): \mathbb{R}^{D}\rightarrow\mathbb{R}^{M}$ are two learnable multi-layer perceptrons (MLPs) for inter-event time and event mark reconstruction, respectively. The $\operatorname{Softplus}$ activation function ensures the non-negativity of the reconstructed inter-event times. Here, $\hat{\mathbf{p}}^{\ell}[m]$ denotes the $m$-th element of the probability vector $\hat{\mathbf{p}}^{\ell}\in (0,1)^M$, i.e., the predicted probability assigned to the mark $m$.

\vspace{0.2cm}
Importantly, we train this event decoder to reconstruct event sequences accurately, such that latent representations generated by the reverse denoising process can be mapped to high-quality event sequences at inference time. We therefore define the reconstruction loss using the mean squared error for inter-event times and the cross-entropy loss for one-hot encoded marks:
\begin{equation}\label{equ:reconstruction loss}
\begin{aligned}
\mathcal{L}_{\text {Recon}}(\mathbf{x}, \hat{\mathbf{x}};\mathbf{\phi})&=\frac{1}{L} \sum_{b=1}^B\mathcal{L}_{\text {Recon}}^b(\mathbf{x}^b, \hat{\mathbf{x}}^b;\mathbf{\phi})\\
&=\frac{1}{L} \sum_{b=1}^B\sum_{\ell=\ell_b+1}^{\ell_{b+1}}\Big(\!\left(\tau^{\ell}-\hat{\tau}^{\ell}\right)^2 \!- \log \hat{\mathbf{p}}^{\ell}[m^{\ell}]\Big).
\end{aligned}
\end{equation}

As previously mentioned, the event encoder does not contain learnable parameters; thus, optimizing the reconstruction loss $\mathcal{L}_{\text {Recon}}$ only updates the event decoder parameters $\mathbf{\phi}$.

\begin{algorithm}[t]
\caption{LBDTPP Sampling}
\label{alg:unconditional sampling}
\textbf{Input:} model $\mathbf{z}_{\mathbf{\theta}}$, generation interval $[0,T]$

$\mathbf{x}, \mathbf{K}, \mathbf{V} \leftarrow \emptyset$; \,$b=1$, $\ell_{b}=0$, $t^{\ell_{b}}=0$

\textbf{while} $t^{\ell_{b}}<T$ \textbf{do}
\vspace{-1pt}
\begin{enumerate}[label=\arabic*.]
    \item $\mathbf{z}^b \leftarrow \texttt{SAMPLE}\left(\mathbf{z}_\theta^b, \mathbf{K}^{1: b-1}, \mathbf{V}^{1: b-1}\right)\, \hfill \triangleright \text{\small$\texttt{len}(\mathbf{z}^b)=L^{\prime}$}$ \vspace{3pt}
    \item $\emptyset, \mathbf{K}^b, \mathbf{V}^b \leftarrow \mathbf{z}_\theta^b\left(\mathbf{z}^b\right) \hfill \triangleright$ \text{\small $\mathrm{KV}$ cache} \vspace{3pt} 
    \item $(\mathbf{K}, \mathbf{V}) \leftarrow\left(\mathbf{K}^{1: b-1} \oplus \mathbf{K}^b, \mathbf{V}^{1: b-1} \oplus \mathbf{V}^b\right)$ \vspace{3pt}
    \item $\mathbf{x}^b=\{(\tau^{\ell_{b}+i},m^{\ell_{b}+i})\}_{i=1}^{L^{\prime}} \leftarrow \texttt{Decoder}(\mathbf{z}^b)$\vspace{3pt}
    \item $\mathbf{x} \leftarrow \mathbf{x}^{1: b-1} \oplus \mathbf{x}^b$ \vspace{3pt}
    \item $\ell_{b+1}=\ell_{b}+\text{\small $L^{\prime}$}$;\,\,$t^{\ell_{b+1}}=t^{\ell_{b}}+\sum_{i=1}^{L^{\prime}}\tau^{\ell_{b}+i}$  \vspace{3pt} 
    \item $b \leftarrow b+1$  \vspace{-2pt}
\end{enumerate}
\textbf{end}\vspace{-1pt}\\
\textbf{return} $\text{truncate}(\mathbf{x},T)  \hfill \triangleright$ \text{\small Truncate at $T$ based on timestamps} 
\end{algorithm}

\subsection{End-to-End Training}\label{sec:training}

We train LBDTPP end-to-end by jointly optimizing two objectives: the latent block diffusion loss in \myref{equ:total block diffusion loss} and the reconstruction loss in \myref{equ:reconstruction loss}. The former loss learns the conditional distribution of latent event blocks, while the latter ensures that latent representations can be decoded into events that closely match the original sequence. The overall objective is defined as their weighted combination:

\begin{equation}\label{equ: overall loss}
\mathcal{L}_{\text {Overall}}(\mathbf{x};\mathbf{\theta},\mathbf{\phi}) =\mathcal{L}_{\text {LBD}}+ \lambda \mathcal{L}_{\text {Recon}}.
\end{equation}
Here, the hyperparameter $\lambda>0$ balances the two loss terms.

To efficiently compute the above latent block diffusion loss $\mathcal{L}_{\text {LBD}}$, we pre-calculate keys and values $\mathbf{K}^{1: B}, \mathbf{V}^{1: B}$ for the full event sequence representation $\mathbf{z}$ in a first forward pass $(\emptyset, \mathbf{K}^{1: B}, \mathbf{V}^{1: B}) \leftarrow \mathbf{z}_{\mathbf{\theta}}(\mathbf{z})$, as shown in \cref{alg:unconditional training}, where $\mathbf{K}^{b}$ and $\mathbf{V}^{b}$ correspond to block $b$. We then compute the denoised prediction $\hat{\mathbf{z}}^b$ using $\mathbf{z}_{\mathbf{\theta}}^b(\mathbf{z}_k^{b}, \mathbf{K}^{1:b-1}, \mathbf{V}^{1:b-1}, k_b)$ for each $b\in[B]$, where the cached keys and values from preceding blocks $\mathbf{z}^{<b}$ are utilized. In practice, instead of invoking the denoising network $\mathbf{z}_{\mathbf{\theta}}^b$ in a loop $B$ times, we adopt a vectorized implementation approach by leveraging the block diffusion attention mask \cite{arriolablock}. This specialized attention mask for the concatenation $\mathbf{z}_{\text {noisy}} \oplus \mathbf{z}$ ensures that noisy event representations attend to other noisy event representations in their block and to all clean event representations in preceding blocks, which allows us to compute $\mathcal{L}_{\text{LBD}}$ in a single forward pass on $\mathbf{z}_{\text {noisy}} \oplus \mathbf{z}$. Here, the noisy sequence representation $\mathbf{z}_{\text {noisy}}:=\mathbf{z}_{k_1}^1 \oplus \cdots \oplus \mathbf{z}_{k_B}^B$ is obtained by applying a noise level $k_b$ to each block $\mathbf{z}^b$ based on \myref{equ:closed form}.

\subsection{Unconditional Generation}\label{sec:Unconditional Generation}

During inference, we sample one block of $L^{\prime}$ latent event representations at each time, conditioned on the previously sampled blocks, with the keys and values cached to avoid redundant computations. Note that any diffusion sampling algorithm $\texttt{SAMPLE}$, such as DDPM \cite{ho2020denoising} or DDIM \cite{songdenoising}, can be utilized to perform the reverse process in Eqs.~\eqref{eq:p_joint} and \eqref{equ:reverse definition}, ultimately obtaining a sample $\mathbf{z}^b$ from $p_{\mathbf{\theta}}\left(\mathbf{z}^{b} \mid \mathbf{z}^{<b}\right)$. Then, the latent block $\mathbf{z}^b$ is decoded to obtain the corresponding $L^{\prime}$ events $\mathbf{x}^b$. Repeat this process until the generated sequence reaches the termination time $T$, and then truncate the sequence at this point, as the event timestamps from the last block may exceed $T$. We summarize this unconditional sampling procedure in \cref{alg:unconditional sampling}.

Crucially, the sampling procedure of LBDTPP offers two advantages: (i) it generates multiple high-quality events simultaneously via Gaussian diffusion within blocks, which can reduce error accumulation caused by one-by-one generation in autoregressive TPPs; and (ii) it enables variable-length event sequence generation in a block-by-block manner, overcoming the fixed-length generation limitation of non-autoregressive diffusion TPPs.

\subsection{Conditional Generation}\label{sec:Conditional Generation}

Next, we extend the unconditional modeling and sampling methods introduced above to the context of conditional generation. We re-express an event sequence in the interval $[0,T]$ as $\mathbf{x}=(\mathbf{x}_{c}, \mathbf{x}_{u})$, where $\mathbf{x}_{c}=(\mathbf{x}^{1}, \ldots, \mathbf{x}^{L_{c}})$ represents the historical events in $[0,T_{c}]$, $\mathbf{x}_{u}=(\mathbf{x}^{L_{c}+1}, \ldots, \mathbf{x}^{L})$ represents the future events in $(T_{c},T]$, and $0<T_{c}<T$. The goal of conditional generation is to generate the future sequence $\mathbf{x}_{u}$ based on the historical sequence $\mathbf{x}_{c}$.

To achieve this, we first encode the entire sequence $\mathbf{x}$ into its latent representation $\mathbf{z}=(\mathbf{z}_{c}, \mathbf{z}_{u})$ using the event encoder described in \cref{sec:model architecture}, where $\mathbf{z}_{c}$ and $\mathbf{z}_{u}$ correspond to the historical and future sequence representations, respectively. We then partition $\mathbf{z}_{u}$ into $B_{u}:=L_{u}/L^{\prime}$ blocks of length $L^{\prime}$, where $L_{u}=L-L_{c}$. Similar to \myref{equ:block factorization}, we factorize the log-likelihood of the future sequence representation $\mathbf{z}_{u}$ conditioned on the historical representation $\mathbf{z}_{c}$ as:
\begin{equation}\label{equ:conditional block factorization}
\log p_{\mathbf{\theta}}(\mathbf{z}_{u} \mid \mathbf{z}_{c})=\sum_{b_u=1}^{B_{u}} \log p_{\mathbf{\theta}}\left(\mathbf{z}_u^{b_u} \mid \mathbf{z}_{c}, \mathbf{z}_u^{<b_u}\right),
\end{equation}
where $\mathbf{z}_u^{b_u}$ denotes the $b_u$-th block of $\mathbf{z}_{u}$, and $\mathbf{z}_u^{<b_u}$ represents all historical blocks before block $b_u$ within $\mathbf{z}_{u}$. Each conditional distribution $p_{\mathbf{\theta}}\left(\mathbf{z}_u^{b_u} \mid \mathbf{z}_{c}, \mathbf{z}_u^{<b_u}\right)$ is modeled by Gaussian diffusion over block $b_u$, similar to Eqs.~\eqref{eq:q_joint}--\eqref{equ:forward definition} and Eqs.~\eqref{eq:p_joint}--\eqref{equ:reverse definition}. The training algorithm for conditional generation closely follows that of unconditional generation. Specifically, both of them share the same reconstruction loss. Besides, the latent block diffusion loss of conditional generation can be seen as a part of the summation term in \myref{equ:total block diffusion loss}, since the log-likelihood of conditional generation in \myref{equ:conditional block factorization} is a component of that of unconditional generation in \myref{equ:block factorization}, i.e., $\log p_{\mathbf{\theta}}(\mathbf{z})=\log p_{\mathbf{\theta}}(\mathbf{z}_{c})+\log p_{\mathbf{\theta}}(\mathbf{z}_{u} \mid \mathbf{z}_{c})$.

For conditional generation, unlike unconditional generation where the entire sequence is sampled from scratch, we only need to generate the future sequence $\mathbf{x}_{u}$ given the historical sequence $\mathbf{x}_{c}$. Similar to \cref{alg:unconditional sampling}, this is achieved by sequentially sampling the future block $\mathbf{z}_u^{b_u}$ through the diffusion sampling algorithm $\texttt{SAMPLE}$, conditioned on both the historical representation $\mathbf{z}_{c}$ and the previously sampled future blocks $\mathbf{z}_u^{<b_u}$. Subsequently, we decode $\mathbf{z}_u^{b_u}$ to obtain the future events $\mathbf{x}_u^{b_u}$. This procedure is repeated until the termination time $T$, where the sequence is then truncated.

\subsection{Theoretical Analysis}\label{sec:theoretical analysis}

We now provide a theoretical analysis under the assumptions stated below to illustrate how block-wise generation can reduce error accumulation compared with event-wise generation in unconditional generation. The conditional case follows similarly by fixing the observed history. We measure error accumulation by the Wasserstein discrepancy between the generated latent sequence distribution and the true latent sequence distribution. We consider a length-$L$ latent sequence generated from scratch by an unconditional model, and assume $L=BL^{\prime}$ after padding if necessary. The stopping rule based on the termination time $T$ only determines how many generated events are retained, and is thus orthogonal to the prefix-level error accumulation analyzed below. We emphasize prefix-level accumulation because each transition is conditioned on previously generated events or blocks, so discrepancies in the generated prefix can perturb the conditioning context and propagate to subsequent transitions. 

For two latent sequences $\mathbf{u},\mathbf{v}\in\mathbb{R}^{r\times D}$, define the additive sequence metric
\begin{equation}
d_r(\mathbf{u},\mathbf{v})=\sum_{\ell=1}^{r}\|\mathbf{u}^{\ell}-\mathbf{v}^{\ell}\|_2,
\end{equation}
and let $W_1^r(\cdot,\cdot)$ denote the Wasserstein-1 distance induced by $d_r$. Denote by $\mathsf{P}_{1:L}$ the true unconditional latent distribution of the length-$L$ sequence, by $\mathsf{Q}_{1:L}^{\mathrm{AR}}$ the distribution generated from scratch by an event-wise autoregressive model, and by $\mathsf{Q}_{1:L}^{\mathrm{BL}}$ the distribution generated from scratch block by block.

\begin{assumption}[Uniform local approximation and prefix stability]\label{assump:error_accumulation}
For event-wise autoregressive generation, let $\mathsf{P}_{\ell}(\cdot\mid \mathbf{z}^{<\ell})$ and $\mathsf{Q}_{\ell}^{\mathrm{AR}}(\cdot\mid \mathbf{z}^{<\ell})$ denote the true and learned one-event transition kernels in the chain-rule factorization of the unconditional sequence distribution. There exist constants $\varepsilon_{\mathrm{AR}}\geq0$ and $\rho_{\mathrm{AR}}\geq0$ such that, for all $\ell$ and event-prefix realizations $\mathbf{h},\mathbf{h}'\in\mathbb{R}^{(\ell-1)\times D}$,
\begin{align}
W_1^1\!\left(\mathsf{P}_{\ell}(\cdot\mid\mathbf{h}),\mathsf{Q}_{\ell}^{\mathrm{AR}}(\cdot\mid\mathbf{h})\right)&\leq \varepsilon_{\mathrm{AR}},\\
W_1^1\!\left(\mathsf{P}_{\ell}(\cdot\mid\mathbf{h}),\mathsf{P}_{\ell}(\cdot\mid\mathbf{h}')\right)&\leq \rho_{\mathrm{AR}} d_{\ell-1}(\mathbf{h},\mathbf{h}').
\end{align}
For block-wise unconditional generation, let $\mathsf{P}_{b}^{\mathrm{BL}}(\cdot\mid \mathbf{z}^{<b})$ and $\mathsf{Q}_{b}^{\mathrm{BL}}(\cdot\mid \mathbf{z}^{<b})$ denote the true and learned transition kernels for the $b$-th block in the block factorization of the unconditional sequence distribution. There exist constants $\varepsilon_{\mathrm{BL}}\geq0$ and $\rho_{\mathrm{BL}}\geq0$ such that, for all $b$ and block-prefix realizations $\mathbf{g},\mathbf{g}'\in\mathbb{R}^{(b-1)L^{\prime}\times D}$,
\begin{align}
W_1^{L^{\prime}}\!\left(\mathsf{P}_{b}^{\mathrm{BL}}(\cdot\mid\mathbf{g}),\mathsf{Q}_{b}^{\mathrm{BL}}(\cdot\mid\mathbf{g})\right)&\leq \varepsilon_{\mathrm{BL}},\\
W_1^{L^{\prime}}\!\left(\mathsf{P}_{b}^{\mathrm{BL}}(\cdot\mid\mathbf{g}),\mathsf{P}_{b}^{\mathrm{BL}}(\cdot\mid\mathbf{g}')\right)&\leq \rho_{\mathrm{BL}} d_{(b-1)L^{\prime}}(\mathbf{g},\mathbf{g}').
\end{align}
\end{assumption}

The following theorem formalizes the comparison between event-wise and block-wise generation. Its purpose is to isolate the accumulation mechanism rather than to assert that block-wise generation always has a smaller local approximation error: when the block-level approximation and stability are comparable to their event-wise counterparts, reducing the number of recursive transitions from $L$ events to $B$ blocks reduces the opportunities for prefix-level errors to propagate.

\begin{theorem}[Prefix-level generation-error accumulation in the unconditional setting]\label{thm:error_accumulation}
Under \Cref{assump:error_accumulation}, define
\begin{equation}
A_n(\rho)=
\begin{cases}
\frac{(1+\rho)^n-1}{\rho}, & \rho>0,\\
n, & \rho=0.
\end{cases}
\end{equation}
Then the event-wise autoregressive generator satisfies
\begin{equation}\label{eq:ar_error_bound}
W_1^L\!\left(\mathsf{P}_{1:L},\mathsf{Q}_{1:L}^{\mathrm{AR}}\right)
\leq \varepsilon_{\mathrm{AR}} A_L(\rho_{\mathrm{AR}}),
\end{equation}
whereas the block-wise generator satisfies
\begin{equation}\label{eq:block_error_bound}
W_1^L\!\left(\mathsf{P}_{1:L},\mathsf{Q}_{1:L}^{\mathrm{BL}}\right)
\leq \varepsilon_{\mathrm{BL}} A_B(\rho_{\mathrm{BL}}).
\end{equation}
Furthermore, if $\varepsilon_{\mathrm{BL}}\leq L^{\prime}\varepsilon_{\mathrm{AR}}$ and $\rho_{\mathrm{BL}}\leq \rho_{\mathrm{AR}}=\rho$, then the relative block-wise accumulation factor, i.e., the ratio between the block-wise and event-wise accumulation upper bounds under these conditions, is bounded by
\begin{equation}\label{eq:accumulation_factor}
\frac{L^{\prime}A_B(\rho)}{A_L(\rho)}\leq 1,
\end{equation}
with strict inequality when $\rho>0$ and $L^{\prime}>1$.
\end{theorem}

\Cref{thm:error_accumulation} makes the generation-error accumulation explicit. The bound in \myref{eq:ar_error_bound} shows that, for event-wise autoregressive generation, the local one-event approximation error $\varepsilon_{\mathrm{AR}}$ is multiplied by an accumulation factor $A_L(\rho_{\mathrm{AR}})$ over all $L$ event-level sampling steps. The bound in \myref{eq:block_error_bound} is the block-wise counterpart: the local block approximation error $\varepsilon_{\mathrm{BL}}$ is multiplied by $A_B(\rho_{\mathrm{BL}})$ over only $B=L/L^{\prime}$ block-level sampling steps. Since $A_n(\rho)$ is nondecreasing in $n$ for $\rho\geq 0$ and $B\leq L$, block-wise generation has a shorter accumulation horizon than event-wise generation. Within each block, the $L^{\prime}$ latent event representations are sampled simultaneously by diffusion rather than being recursively fed back one by one, so errors made for earlier events inside the same block do not enter as prefix-level distributional discrepancies for later events in that block.

The comparison in \myref{eq:accumulation_factor} relies on two conditions. The condition $\varepsilon_{\mathrm{BL}}\leq L^{\prime}\varepsilon_{\mathrm{AR}}$ means that learning one block is no worse, in distributional error, than accumulating the $L^{\prime}$ corresponding event-wise local errors. The condition $\rho_{\mathrm{BL}}\leq \rho_{\mathrm{AR}}=\rho$ assumes that the block transition is at least as stable as the event-wise transition with respect to prefix perturbations. Under these comparable-error and comparable-stability conditions, the ratio in \myref{eq:accumulation_factor} is no larger than one, meaning that the block-wise error accumulation upper bound is no larger than its event-wise counterpart. The ratio is strictly smaller when $\rho>0$ and $L^{\prime}>1$. This is consistent with the empirical trade-off observed in \Cref{sec:different block sizes}: increasing $L^{\prime}$ initially improves generation performance by reducing event-by-event recursion, whereas overly large blocks may degrade performance because the block transition distribution becomes higher-dimensional and harder to learn.

If the decoder $g_{\phi}$ is $L_{\mathrm{dec}}$-Lipschitz from $(\mathbb{R}^{L\times D},d_L)$ to an event-space discrepancy $\Delta_L$, then the latent-space result can be transferred to decoded event sequences: for either generator $\mathsf{Q}\in\{\mathsf{Q}^{\mathrm{AR}},\mathsf{Q}^{\mathrm{BL}}\}$,
\begin{equation}\label{eq:decoder_transfer}
W_{\Delta_L}\!\left((g_{\phi})_{\#}\mathsf{P}_{1:L},(g_{\phi})_{\#}\mathsf{Q}_{1:L}\right)
\leq L_{\mathrm{dec}} W_1^L\!\left(\mathsf{P}_{1:L},\mathsf{Q}_{1:L}\right),
\end{equation}
where $(g_{\phi})_{\#}$ denotes the push-forward distribution. Therefore, reducing the latent distributional discrepancy directly reduces decoded event sequence discrepancy up to the decoder Lipschitz constant. The proof is provided in the supplementary material.

\section{Experiments}\label{sec:experiments}

In this section, we conduct extensive experiments to evaluate the performance of LBDTPP, addressing the following major research questions (RQs):
\begin{itemize}[leftmargin=*]
    \item \textbf{RQ1:} How does LBDTPP perform compared to state-of-the-art TPP baselines for the unconditional generation task?
    \item \textbf{RQ2:} How does LBDTPP perform compared to state-of-the-art TPP baselines for the conditional generation task? 
    \item \textbf{RQ3:} What are the sources of performance improvement of LBDTPP over autoregressive and non-autoregressive TPP baselines, and how does the block size affect its performance?
    \item \textbf{RQ4:} Can LBDTPP capture the empirical distributions of event timestamps and marks?
    \item \textbf{RQ5:} How sensitive is LBDTPP with respect to different hyperparameters?
    \item \textbf{RQ6:} How do different model variants, such as utilizing a learnable mark encoder or adopting $\boldsymbol{\epsilon}^b$-prediction, affect the performance of LBDTPP?
    \item \textbf{RQ7:} How does the sampling time of LBDTPP compare to that of baseline models? 
\end{itemize}

\subsection{Experimental Setup}
\textbf{Datasets.} We use six real-world benchmark datasets containing event sequences from multiple domains. \Cref{table:statistics of datasets} summarizes the statistics of these datasets. All datasets are available at the CDiff repository \cite{zeng2024interacting}.

\begin{itemize}[leftmargin=*]
    \item \textbf{Taxi} \cite{taxi} contains time-stamped taxi pick-up and drop-off events throughout the five boroughs of New York city. Each combination of borough, whether it's a pick-up or drop-off, defines a mark, resulting in a total of 10 marks.
    \item \textbf{Taobao} \cite{zhu2018learning} includes time-stamped user click behaviors on the Taobao platform. Each user has a sequence of product click events, where each event contains a timestamp and a product category.
    \item \textbf{StackOverflow} \cite{so} contains user-awarded collections from a question-answering website. Each user is awarded a sequence of badges, with a total of 22 different badge marks.
    \item \textbf{Retweet} \cite{zhou2013learning} consists of sequences of time-stamped user retweet events, categorized into three marks based on the users' following sizes: ``small", ``medium", and ``large".
    \item \textbf{MOOC} \cite{kumar2019predicting} includes records of student interactions within an online course platform. Each type of interaction (e.g., video watching, forum posting) is treated as a distinct mark. We use the same pre-processing approach as in \cite{bosser2023predictive,zeng2024interacting}.
    \item \textbf{Amazon} \cite{ni2019justifying} contains time-stamped user product review behaviors where product categories are seen as event marks.
\end{itemize}

The data pre-processing for the two generation tasks is as follows. (i) For unconditional generation (\Cref{sec:unconditional generation task}), event timestamps within each sequence are normalized to a unified scale by dividing them by the maximum termination time, defined as the maximum last timestamp in the training set. During evaluation, the generated timestamps are mapped back to the original scale for comparison with the test sequences. (ii) For conditional generation (\Cref{sec:conditional generation task}), future events are predicted directly based on historical events in their natural time scale, so no explicit time normalization is applied.

\begin{table}[t]
\renewcommand{\arraystretch}{1}
  \caption{Statistics of each dataset}
  \vspace{-2pt}
  \label{table:statistics of datasets}
  \centering
  \adjustbox{width=0.49\textwidth}{
\begin{tabular}{lrrrrrrr}
\toprule
\multirow{2.5}{*}{Dataset} & \multirow{2.5}{*}{$M$} & \multicolumn{3}{c}{\# Sequences} & \multicolumn{3}{c}{Sequence Length} \\ \cmidrule(lr){3-5} \cmidrule(lr){6-8} 
                         &                           & Train     & Dev       & Test     & Min        & Mean       & Max     \\ \midrule
Taxi                  &        10                  &   1,400  &   200  &  400  &     36     &     37     &    38     \\
Taobao               &     17                     &  1,300   &  200   &  500   &   32       &    58      &     64     \\
StackOverflow                     &            22              & 1,401    & 401    &  401  &    41      &     65     &     101     \\
Retweet                   &   3                       &   6,500  &   731  &  500  &   30       &     53     &     97     \\
MOOC            &      50                    &  1,321   &  333   &  341  &    21      &  45        &   234      \\
Amazon                 &       16                   &   5,200  &  300   &  500  &      30    &        50  &   94       \\ \bottomrule
\end{tabular}}
\end{table}

\textbf{Baselines.} We compare our model with nine TPP baselines, including both autoregressive and non-autoregressive TPPs.

\begin{itemize}[leftmargin=*]
\item \textbf{NHP} \cite{mei2017neural} introduces a continuous-time long short-term memory (LSTM) network for modeling event sequences. The conditional intensity function of NHP is capable of decaying over time, allowing it to flexibly capture the influence of past events on future event occurrences. 

\item \textbf{LNM} \cite{shchurintensity} models the conditional density distribution of TPPs using a log-normal mixture model. This intensity-free approach allows for a more flexible and efficient representation of the temporal dynamics, providing advantages in terms of both expressiveness and ease of sampling.

\item \textbf{THP} \cite{zuo2020transformer} incorporates a Transformer architecture to model the conditional intensity function of TPPs. By leveraging self-attention mechanisms, THP captures long-range dependencies in event sequences, enabling it to model intricate temporal patterns effectively.

\item \textbf{AttNHP} \cite{yang2022transformer} utilizes a Transformer architecture to model event sequences, learning rich embeddings of actual and possible events at any given time, based on lower-level representations of these events and their context.

\item \textbf{S2P2} \cite{chang2025deep} adapts deep state-space models to marked TPPs by combining neural jump stochastic differential equations with nonlinear transformations. This architecture allows the model to efficiently capture continuous-time dynamics and long-range dependencies in event sequences.

\item \textbf{DualTPP} \cite{deshpande2021long} combines two components for long-horizon event forecasting: an autoregressive TPP model that captures short-term event dynamics at a microscopic level and a count model that handles the macroscopic, long-term behavior.

\item \textbf{HYPRO} \cite{xue2022hypro} introduces a hybridly normalized probabilistic model designed for long-horizon prediction of event sequences. It combines an autoregressive base model with an energy function, which reweights the predicted sequences to improve the realism of long-term forecasts.

\item \textbf{TCDDM} \cite{linexploring} designs a generative framework for neural TPPs that adopts a diffusion-based probabilistic decoder. This approach enhances predictive performance by leveraging diffusion models to generate high-quality inter-event times.

\item \textbf{CDiff} \cite{zeng2024interacting} proposes to address the task of long-horizon event forecasting by employing interacting diffusion processes. It introduces two coupled diffusion processes, one for event marks and one for inter-event times, which interact through their respective denoising functions.

\end{itemize}

The specific baseline settings for the two generation tasks are as follows. (i) For unconditional generation, we utilize the EasyTPP benchmark \cite{xue2023easytpp} to evaluate five autoregressive TPP baselines: \textbf{NHP}, \textbf{LNM}, \textbf{THP}, \textbf{AttNHP}, and \textbf{S2P2}. We do not compare against the other four baselines (i.e., \textbf{DualTPP}, \textbf{HYPRO}, \textbf{TCDDM}, and \textbf{CDiff}), because they are specifically designed for conditional generation and cannot be applied to variable-length unconditional generation without significant modifications. (ii) For conditional generation, we compare our model with all nine TPP baselines. We evaluate \textbf{THP} and \textbf{S2P2} with the EasyTPP benchmark, while the results of the other seven baselines are taken from \cite{zeng2024interacting}.

\begin{table*}[thb!]
    \centering
    \renewcommand{\arraystretch}{0.98}
    \caption{$\textbf{OTD}$ and $\textbf{RMSE}_{m}$ of unconditional generation reported in mean $\pm$ s.d. \!\colorbox{bestcell}{Best} and \!\colorbox{secondcell}{second best} are highlighted}
    \vspace{-3pt}
    \label{table: unconditional generation OTD and RMSE comparison}
    
    \adjustbox{width=0.9\textwidth}{
    \begin{tabular}{lcc|cc|cc} \toprule
    &\multicolumn{2}{c}{\textbf{Taxi}}&\multicolumn{2}{c}{\textbf{Taobao}}&\multicolumn{2}{c}{\textbf{StackOverflow}}\\ 
        & $\textbf{OTD}$ & $\textbf{RMSE}_{m}$ & $\textbf{OTD}$ & $\textbf{RMSE}_{m}$& $\textbf{OTD}$ & $\textbf{RMSE}_{m}$\\ \midrule

          	\textbf{NHP}        & 67.897 $\pm$ 0.115 & \cellcolor{secondcell}5.499 $\pm$ 0.030 & 118.233 $\pm$ 0.415 & \cellcolor{secondcell}7.329 $\pm$ 0.018 & 161.867 $\pm$ 0.378 & 6.083 $\pm$ 0.027 \\  
          	\textbf{LNM}        & \cellcolor{secondcell}58.004 $\pm$ 0.127 & 6.619 $\pm$ 0.021 & 106.362 $\pm$ 0.113 & 8.483 $\pm$ 0.009 & \cellcolor{secondcell}140.620 $\pm$ 0.510 & 7.609 $\pm$ 0.048 \\
            \textbf{THP}        & 63.449 $\pm$ 0.074 & 6.867 $\pm$  0.012 & 124.732 $\pm$ 0.462  & 7.349 $\pm$ 0.033  &  162.674 $\pm$ 0.347  & 7.729 $\pm$ 0.055   \\
            \textbf{AttNHP}     & 71.618 $\pm$ 0.334 & 5.762 $\pm$ 0.009 & 129.901 $\pm$ 0.503 & 7.545 $\pm$ 0.086 & 166.370 $\pm$ 0.449 & 6.373 $\pm$ 0.049 \\
            \textbf{S2P2}        & 66.264 $\pm$ 0.230 & 6.284 $\pm$ 0.042 & \cellcolor{secondcell}104.277 $\pm$ 0.388 & 7.513 $\pm$ 0.023 & 150.335 $\pm$ 0.249 & \cellcolor{secondcell}5.579 $\pm$ 0.050   \\
          	\midrule      
        	\textbf{LBDTPP}      & \cellcolor{bestcell}39.458 $\pm$ 0.843 & \cellcolor{bestcell}2.985 $\pm$ 0.196 & \cellcolor{bestcell}98.965 $\pm$ 0.608 & \cellcolor{bestcell}7.192 $\pm$ 0.069 & \cellcolor{bestcell}128.427 $\pm$ 0.981 & \cellcolor{bestcell}5.007 $\pm$ 0.081 \\ \bottomrule

        \addlinespace[2pt]
         &\multicolumn{2}{c}{\textbf{Retweet}}&\multicolumn{2}{c}{\textbf{MOOC}}&\multicolumn{2}{c}{\textbf{Amazon}}\\ 
        & $\textbf{OTD}$ & $\textbf{RMSE}_{m}$ & $\textbf{OTD}$ & $\textbf{RMSE}_{m}$& $\textbf{OTD}$ & $\textbf{RMSE}_{m}$\\ \midrule

          	\textbf{NHP}        & \cellcolor{secondcell}96.343 $\pm$ 0.233 & 20.610 $\pm$ 0.063 & 110.819 $\pm$ 0.307 & 2.895 $\pm$ 0.021 & 90.095 $\pm$ 0.160 & 5.591 $\pm$ 0.018 \\  
          	\textbf{LNM}        & 96.695 $\pm$ 0.594 & 21.270 $\pm$ 0.168 & \cellcolor{secondcell}72.895 $\pm$ 0.215 & \cellcolor{secondcell}1.480 $\pm$ 0.003 & 87.689 $\pm$ 0.171 & \cellcolor{secondcell}5.219 $\pm$ 0.026 \\
            \textbf{THP}        & 97.039 $\pm$ 0.144 & 21.448 $\pm$ 0.037  & 116.870 $\pm$ 0.209  & 4.318 $\pm$ 0.015  & 88.990 $\pm$ 0.097  & 5.497 $\pm$ 0.013   \\
            \textbf{AttNHP}     & 96.624 $\pm$ 0.333 & \cellcolor{secondcell}20.575 $\pm$ 0.057 & 116.978 $\pm$ 0.377 & 4.345 $\pm$ 0.029 & \cellcolor{secondcell}87.238 $\pm$ 0.133 & 5.565 $\pm$ 0.006 \\ 
            \textbf{S2P2}        & 97.502 $\pm$ 0.268 & 20.955 $\pm$ 0.010 & 114.752 $\pm$ 0.768 & 2.887 $\pm$ 0.043 & 87.553 $\pm$ 0.231 &  5.348 $\pm$ 0.019  \\
          	\midrule      
        	\textbf{LBDTPP}      & \cellcolor{bestcell}95.790 $\pm$ 0.754 & \cellcolor{bestcell}20.402 $\pm$ 0.252 & \cellcolor{bestcell}64.943 $\pm$ 0.266 & \cellcolor{bestcell}1.407 $\pm$ 0.007 & \cellcolor{bestcell}80.914 $\pm$ 0.564 & \cellcolor{bestcell}5.008 $\pm$ 0.007 \\ \bottomrule
    \end{tabular}}
\end{table*}
\textbf{Evaluation Metrics.} We adopt four commonly used metrics to evaluate the quality of the generated marked event sequences. In this work, we do not report density-based metrics such as log-likelihood because several diffusion TPP baselines and our LBDTPP model, are implicit generators or operate in latent space, making event-space log-likelihoods unavailable or not directly comparable.

\begin{itemize}[leftmargin=*]
    
    \item \textbf{OTD}: The optimal transport distance between two marked event sequences, which defines the minimum cost required to edit a generated event sequence into the ground truth event sequence, measuring the sequence-level similarity between them. We report the average values of the OTD across different values of the deletion/insertion cost constant $C$: $\{0.05, 0.5, 1, 1.5, 2, 3, 4\}$. More details on this OTD metric can be found in \cite{mei2019imputing}.

    \item $\textbf{RMSE}_{m}$: The root mean square error of the number of events for each mark, which quantifies how well the event mark distribution of the generated sequence matches that of the ground truth sequence. For $m\in[M]$, we compute the count of events of mark $m$ in the generated sequence, $\hat{C}_m$, and the true sequence, $C_m$. This metric is then calculated as $\sqrt{\frac{1}{M} \sum_{m=1}^M(C_m-\hat{C}_m)^2}$.
    
    \item $\textbf{RMSE}_{\tau}$: The root mean square error between the inter-event times of the generated sequence and those of the ground truth sequence, which assesses the temporal accuracy of the generated events. For a generated sequence with inter-event times $\hat{\tau}_1, \ldots, \hat{\tau}_{H}$ and a ground truth sequence with inter-event times $\tau_1, \ldots, \tau_{H}$, we compute $\text{RMSE}_{\tau}=\sqrt{\frac{1}{H} \sum_{i=1}^{H}(\tau_i-\hat{\tau}_i)^2}$, where $H$ is the number of inter-event times.
    \item \textbf{sMAPE}: The symmetric mean absolute percentage error between the inter-event times of the generated sequence and those of the ground truth sequence, which also evaluates the temporal accuracy. It is defined as $\text{sMAPE}=\frac{100}{H} \sum_{i=1}^{H} \frac{2|\tau_i-\hat{\tau}_i|}{|\tau_i|+|\hat{\tau}_i|}$.
\end{itemize}

Since $\textbf{RMSE}_{\tau}$ and \textbf{sMAPE} metrics require the generated and ground-truth sequences to have the same number of events, they are evaluated only in the fixed-length conditional generation benchmark \cite{zeng2024interacting} (\Cref{sec:conditional generation task}), where the number of future events is fixed to $H=20$. In contrast, \textbf{OTD} and $\textbf{RMSE}_{m}$ are evaluated in both unconditional and conditional generation benchmarks, as they can be computed between sequences with either different or equal length. For each run, all metrics are averaged over the full test set. Each model is trained with 10 different random seeds, and the mean and standard deviation (s.d.) of each metric are reported in this work.

\textbf{Implementation Details.} As described in \Cref{sec:model architecture}, the event encoder contains no learnable parameters, the clean latent block $\textbf{z}^b$ is predicted by a Transformer equipped with the block diffusion attention mask \cite{arriolablock}, and the time and mark decoders are implemented as two-layer MLPs. For both unconditional and conditional generation tasks, we set the forward diffusion steps to $K=100$ and use the DDIM sampler \cite{songdenoising} with $S=50$ sampling steps to accelerate sampling. All experiments are conducted on an NVIDIA GeForce RTX 3090 GPU with 24GB of memory. We implement our model using PyTorch \cite{paszke2017automatic}.

\textbf{Training Details.} We set the maximum number of epochs to 50 for all experiments and evaluate the model every 2 epochs on the validation set, selecting the one with the best performance for testing. For both training and evaluation, the batch size is fixed at 32 for all datasets. The Adam optimizer \cite{kingma2014adam} is employed for model optimization.

\textbf{Hyperparameter Setting.} We perform grid search to determine the hyperparameters for LBDTPP based on the validation set. Specifically, we tune the learning rate from $\{0.0001, 0.0005, 0.001, 0.005, 0.01\}$, the embedding dimension $D$ from $\{16,32,64,128\}$, the number of attention heads from $\{1, 2, 4\}$, and the number of Transformer layers from $\{1, 2, 4\}$. In our initial experiments, setting the weight $\lambda$ in the loss function \myref{equ: overall loss} to $1$ already yields stronger performance than the baselines, so we fix this value in all main experiments. We also analyze the sensitivity of $\lambda$ in \Cref{sec:different sampling steps}.

\subsection{Unconditional Generation Task (\textbf{A1})}\label{sec:unconditional generation task}
We first evaluate the performance of our LBDTPP model on the unconditional generation task, where the objective is to generate new event sequences that align well with the underlying data distribution. This task serves as a fundamental benchmark for assessing the fitting capacity and generation quality of different TPP models. A model with strong performance in this setting can also benefit downstream applications, such as system simulation and data augmentation \cite{kerrigan2024eventflow}. Compared with prior settings, unconditional generation for marked event sequences has received limited attention. Existing studies \cite{ludke2023add,kerrigan2024eventflow,zeng2024interacting} have primarily focused on generation of unmarked event sequences in unconditional and conditional settings, or on conditional generation of marked event sequences. We therefore include this setting as an important evaluation scenario to assess whether a model can jointly capture temporal dynamics and mark distributions without historical conditioning.

Since there are no ground truth event sequences available in unconditional generation, we generate event sequences with the same termination times as those in the test set. We then calculate the OTD and $\text{RMSE}_{m}$ metrics between the generated and test sequences to evaluate how well the generated sequences follow the underlying data distribution. Under this setting, the generated sequences and the test sequences generally contain different numbers of events. For all datasets, we set the block size $L^{\prime}=8$ in our LBDTPP model, meaning that each block contains 8 events.

\textbf{Results.} \cref{table: unconditional generation OTD and RMSE comparison} summarizes the OTD and $\text{RMSE}_{m}$ results of unconditional generation. We observe that LBDTPP consistently outperforms all five TPP baselines across six real-world datasets in terms of both metrics. This shows the superior capability of LBDTPP in capturing the underlying distribution of event sequences. Moreover, the performance gap between LBDTPP and these autoregressive TPPs also highlights the effectiveness of our latent block diffusion modeling approach in mitigating error accumulation and generating high-fidelity event sequences.

\begin{table*}[thb!]
    \centering
    \renewcommand{\arraystretch}{1}
    \caption{$\textbf{OTD}$, $\textbf{RMSE}_{m}$, $\textbf{RMSE}_{\tau}$ and \textbf{sMAPE} of conditional generation reported in mean $\pm$ s.d. \colorbox{bestcell}{Best} and \colorbox{secondcell}{second best} are highlighted}
    \vspace{-3pt}
    \label{table: conditional generation OTD and RMSE comparison}
    
    \adjustbox{width=\textwidth}{
    \begin{tabular}{lcccc|cccc} \toprule
    &\multicolumn{4}{c}{\textbf{Taxi}}&\multicolumn{4}{c}{\textbf{Taobao}}\\ 
        & $\textbf{OTD}$ & $\textbf{RMSE}_{m}$ &$\textbf{RMSE}_{\tau}$ & \textbf{sMAPE} & $\textbf{OTD}$ & $\textbf{RMSE}_{m}$&$\textbf{RMSE}_{\tau}$ & \textbf{sMAPE}\\ \midrule
            
          	\textbf{NHP}        & 25.114 $\pm$ 0.268 & 1.297 $\pm$ 0.019 & 0.399 $\pm$ 0.040 & 96.459 $\pm$ 0.521 & 48.131 $\pm$ 0.297 & 3.355 $\pm$ 0.030 & 0.837 $\pm$ 0.009 & 137.644 $\pm$ 0.764\\  
          	\textbf{LNM}        & 24.053 $\pm$ 0.609 & 1.364 $\pm$ 0.032 & 0.384 $\pm$ 0.005 & 95.719 $\pm$ 0.779 & 45.757 $\pm$ 0.287 & 3.193 $\pm$ 0.043 & 0.575 $\pm$ 0.012 & 127.436 $\pm$ 0.606 \\
            \textbf{THP}        & 25.981 $\pm$ 0.286 & 1.983 $\pm$ 0.033 & 0.350 $\pm$ 0.001 & 95.995 $\pm$ 0.410 & 48.776 $\pm$ 0.065 & 3.478 $\pm$ 0.009 & \cellcolor{secondcell}0.429 $\pm$ 0.002 & 162.474 $\pm$ 0.280 \\
            \textbf{AttNHP}     & 24.762 $\pm$ 0.217 & 1.276 $\pm$ 0.015 & 0.430 $\pm$ 0.003 & 97.388 $\pm$ 0.381 & 45.555 $\pm$ 0.345 & 2.737 $\pm$ 0.021 & 0.708 $\pm$ 0.010 & 134.582 $\pm$ 0.920 \\ 
            \textbf{S2P2}        & 21.531 $\pm$ 0.130 & 1.415 $\pm$ 0.022 & \cellcolor{secondcell}0.332 $\pm$ 0.002 & 95.593 $\pm$ 0.274 & 47.797 $\pm$ 0.187 & 3.249 $\pm$ 0.005 & 0.574 $\pm$ 0.011 & 149.758  $\pm$ 0.305\\
            \textbf{DualTPP}    & 24.483 $\pm$ 0.383 & 1.353 $\pm$ 0.037 & 0.402 $\pm$ 0.006 & 95.211 $\pm$ 0.187 & 47.324 $\pm$ 0.541 & 3.237 $\pm$ 0.049 & 0.871 $\pm$ 0.005 & 141.687 $\pm$ 0.431\\
          	\textbf{HYPRO}      & 21.653 $\pm$ 0.163 & 1.231 $\pm$ 0.015 & 0.372 $\pm$ 0.004 & 93.803 $\pm$ 0.454 & \cellcolor{secondcell}44.336 $\pm$ 0.127 & 2.710 $\pm$ 0.021 & 0.594 $\pm$ 0.030 & 134.922 $\pm$ 0.473 \\ 
          	\textbf{TCDDM}      & 22.148 $\pm$ 0.529 & 1.309 $\pm$ 0.030 & 0.382 $\pm$ 0.019 & 90.596 $\pm$ 0.574 & 45.563 $\pm$ 0.889 & 2.850 $\pm$ 0.058 & 0.569 $\pm$ 0.015 & \cellcolor{secondcell}126.512 $\pm$ 0.491 \\      
        	\textbf{CDiff}      & \cellcolor{secondcell}21.013 $\pm$ 0.158 & \cellcolor{secondcell}1.131 $\pm$ 0.017 & 0.351 $\pm$ 0.004 & \cellcolor{secondcell}87.993 $\pm$ 0.178 & 44.621 $\pm$ 0.139 & \cellcolor{secondcell}2.653 $\pm$ 0.022 & 0.551 $\pm$ 0.002 & \cellcolor{bestcell}125.685 $\pm$ 0.151 \\ \midrule      
        	\textbf{LBDTPP}      & \cellcolor{bestcell}19.258 $\pm$ 0.541 & \cellcolor{bestcell}0.993 $\pm$ 0.085 & \cellcolor{bestcell}0.270 $\pm$ 0.003 & \cellcolor{bestcell}74.171 $\pm$ 0.723 & \cellcolor{bestcell}41.377 $\pm$ 0.460 & \cellcolor{bestcell}2.105 $\pm$ 0.036 & \cellcolor{bestcell}0.407 $\pm$ 0.011 & 158.514 $\pm$ 0.841 \\ \bottomrule

        \addlinespace[3pt]
         &\multicolumn{4}{c}{\textbf{StackOverflow}}&\multicolumn{4}{c}{\textbf{Retweet}} \\
         & $\textbf{OTD}$ & $\textbf{RMSE}_{m}$&$\textbf{RMSE}_{\tau}$ & \textbf{sMAPE} & $\textbf{OTD}$ & $\textbf{RMSE}_{m}$&  $\textbf{RMSE}_{\tau}$ & \textbf{sMAPE}\\ \midrule
              	
              	\textbf{NHP} & 43.791 $\pm$ 0.147 & 1.244 $\pm$ 0.030 & 1.487 $\pm$ 0.004 & 116.952 $\pm$ 0.404 & 60.953 $\pm$ 0.079 & 2.651 $\pm$ 0.045 & 27.130 $\pm$ 0.224 & 107.075 $\pm$ 1.398\\ 
              	\textbf{LNM} & 46.280 $\pm$ 0.892 & 1.447 $\pm$ 0.057 & 1.669 $\pm$ 0.005 & 115.122 $\pm$ 0.627 & 61.715 $\pm$ 0.152 & 2.776 $\pm$ 0.043 & 27.582 $\pm$ 0.191 & 106.711 $\pm$ 1.615  \\
                \textbf{THP}        & 54.317 $\pm$ 0.087 & 3.043 $\pm$ 0.001 & 1.413 $\pm$ 0.002 & 115.735 $\pm$ 0.340 & \cellcolor{secondcell}60.417 $\pm$ 0.094 & 3.704 $\pm$ 0.045 & \cellcolor{secondcell}26.432  $\pm$ 0.024 & 134.061 $\pm$ 0.406\\
                \textbf{AttNHP} & 42.591 $\pm$ 0.408 & 1.142 $\pm$ 0.011 & 1.340 $\pm$ 0.006 & 108.542 $\pm$ 0.531 & 60.634 $\pm$ 0.097 & 2.561 $\pm$ 0.054 & 28.812 $\pm$ 0.272 & 107.234 $\pm$ 1.293 \\
                \textbf{S2P2}        & 54.635 $\pm$ 0.066 & 2.912 $\pm$ 0.007& \cellcolor{secondcell}1.151 $\pm$ 0.003& 116.653 $\pm$ 0.129 & 61.496 $\pm$ 0.103& 4.756 $\pm$ 0.044 & 26.796 $\pm$ 0.030 & 140.407 $\pm$ 0.500 \\
                \textbf{DualTPP} & 41.752 $\pm$ 0.200 & \cellcolor{secondcell}1.134 $\pm$ 0.019 & 1.514 $\pm$ 0.017 & 117.582 $\pm$ 0.420 & 61.095 $\pm$ 0.101 & 2.679 $\pm$ 0.026 & 28.914 $\pm$ 0.300 & 106.900 $\pm$ 1.293\\
              	\textbf{HYPRO} & 42.359 $\pm$ 0.170 & 1.140 $\pm$ 0.014 & 1.554 $\pm$ 0.010 & 110.988 $\pm$ 0.559 & 61.031 $\pm$ 0.092 & 2.623 $\pm$ 0.036 & 30.100 $\pm$ 0.413 & 106.110 $\pm$ 1.505\\
              	\textbf{TCDDM} & 42.128 $\pm$ 0.591 & 1.467 $\pm$ 0.014 & 1.315 $\pm$ 0.004 & 107.659 $\pm$ 0.934 & 60.501 $\pm$ 0.087 & \cellcolor{secondcell}2.387 $\pm$ 0.050 & 27.303 $\pm$ 0.152 & \cellcolor{secondcell}106.048 $\pm$ 0.610 \\ 
              	\textbf{CDiff} & \cellcolor{secondcell}41.245 $\pm$ 1.400 & 1.141 $\pm$ 0.007 & 1.199 $\pm$ 0.006 & \cellcolor{secondcell}106.175 $\pm$ 0.340 & 60.661 $\pm$ 0.101 & \cellcolor{bestcell}2.293 $\pm$ 0.034 & 27.101 $\pm$ 0.113 & 106.184 $\pm$ 1.121\\ \midrule
              	\textbf{LBDTPP} & \cellcolor{bestcell}40.969 $\pm$ 0.452 & \cellcolor{bestcell}1.115 $\pm$ 0.010 & \cellcolor{bestcell}1.070 $\pm$ 0.015 & \cellcolor{bestcell}90.975 $\pm$ 0.155 & \cellcolor{bestcell}59.963 $\pm$ 0.632 & 2.503 $\pm$ 0.078 & \cellcolor{bestcell}22.520 $\pm$ 0.305 & \cellcolor{bestcell}89.209 $\pm$ 0.913\\\bottomrule

         \addlinespace[3pt]
          &\multicolumn{4}{c}{\textbf{MOOC}}&\multicolumn{4}{c}{\textbf{Amazon}} \\
         & $\textbf{OTD}$ & $\textbf{RMSE}_{m}$&$\textbf{RMSE}_{\tau}$ & \textbf{sMAPE} & $\textbf{OTD}$ & $\textbf{RMSE}_{m}$&  $\textbf{RMSE}_{\tau}$ & \textbf{sMAPE}\\ \midrule
              	
              	\textbf{NHP} &51.277 $\pm$ 1.768 & 1.458 $\pm$ 0.063 & 0.442 $\pm$ 0.007 & 148.913 $\pm$ 11.628 & 42.571 $\pm$ 0.293 & 2.561 $\pm$ 0.060 & 0.519 $\pm$ 0.023 & 92.053 $\pm$ 1.553 \\ 
              	\textbf{LNM} & 52.890 $\pm$ 1.151 & 1.428 $\pm$ 0.061 & 0.454 $\pm$ 0.008 & 149.987 $\pm$ 16.581 & 43.820 $\pm$ 0.232 & 3.050 $\pm$ 0.286 & 0.481 $\pm$ 0.145 & 90.910 $\pm$ 1.611  \\
                \textbf{THP}        & 68.063 $\pm$ 0.420 & 2.927 $\pm$ 0.070 & 0.530 $\pm$ 0.018 & 178.904 $\pm$ 1.893 & 67.924 $\pm$ 0.018& 5.562 $\pm$ 0.001 & 0.504 $\pm$ 0.002& 128.526 $\pm$ 0.461\\
                \textbf{AttNHP} & 49.121 $\pm$ 0.720 & 1.297 $\pm$ 0.049 & 0.420 $\pm$ 0.009 & 147.756 $\pm$ 4.812 & 39.480 $\pm$ 0.326 & 2.166 $\pm$ 0.026 & 0.476 $\pm$ 0.033 & 84.323 $\pm$ 1.815 \\
                \textbf{S2P2}        & 58.297 $\pm$ 0.281 &  1.287 $\pm$ 0.010 & 0.475 $\pm$ 0.002 & 169.554 $\pm$ 2.246 & 45.546 $\pm$ 0.118 & 3.049 $\pm$ 0.016 & 0.505  $\pm$ 0.005 & 103.734 $\pm$ 0.872 \\
                \textbf{DualTPP} & 50.184 $\pm$ 1.127 & 1.312 $\pm$ 0.019 & 0.435 $\pm$ 0.006 & 147.003 $\pm$ 2.908 & 42.646 $\pm$ 0.752 & 2.562 $\pm$ 0.202 & 0.482 $\pm$ 0.012 & 86.453 $\pm$ 2.044  \\
              	\textbf{HYPRO} & 48.621 $\pm$ 0.352 & 1.169 $\pm$ 0.094 & \cellcolor{secondcell}0.410 $\pm$ 0.005 & \cellcolor{bestcell}143.045 $\pm$ 7.992 & 38.613 $\pm$ 0.536 & \cellcolor{secondcell}2.007 $\pm$ 0.054 & 0.477 $\pm$ 0.010 & 82.506 $\pm$ 0.840  \\
              	\textbf{TCDDM} & 50.739 $\pm$ 0.765 & 1.407 $\pm$ 0.112 & 0.429 $\pm$ 0.015 & \cellcolor{secondcell}145.745 $\pm$ 11.835 & 42.245 $\pm$ 0.174 & 2.998 $\pm$ 0.115 & 0.476 $\pm$ 0.111 & 83.826 $\pm$ 1.508  \\ 
              	\textbf{CDiff} & \cellcolor{secondcell}47.214 $\pm$ 0.628 & \cellcolor{secondcell}1.095 $\pm$ 0.048 & 0.411 $\pm$ 0.009 & 146.361 $\pm$ 14.837 & \cellcolor{bestcell}37.728 $\pm$ 0.199 & 2.091 $\pm$ 0.163 & \cellcolor{secondcell}0.464 $\pm$ 0.086 & \cellcolor{secondcell}81.987 $\pm$ 1.905 \\\midrule
              	\textbf{LBDTPP} & \cellcolor{bestcell}46.633 $\pm$ 1.178 & \cellcolor{bestcell}1.058 $\pm$ 0.009 & \cellcolor{bestcell}0.262 $\pm$ 0.027 & 167.809 $\pm$ 2.039 & \cellcolor{secondcell}38.237 $\pm$ 0.297 & \cellcolor{bestcell}1.822 $\pm$ 0.086 & \cellcolor{bestcell}0.352 $\pm$ 0.019 & \cellcolor{bestcell}77.285 $\pm$ 1.744 \\
         \bottomrule
    \end{tabular}}
\end{table*}
\subsection{Conditional Generation Task (\textbf{A2})}\label{sec:conditional generation task}
Predicting future event occurrences based on historical observations is a crucial task in various real-world applications, including medical diagnosis and financial transactions. Below we compare the forecasting performance of LBDTPP with nine autoregressive and non-autoregressive TPP models. We follow the experimental setup of prior work \cite{zeng2024interacting}, predicting the last 20 events of each sequence based on preceding events (i.e., $H=20$). Metrics are computed by comparing the generated future sequences with the ground truth, and the block size $L^{\prime}$ is set to 4 for all datasets.

\textbf{Results.} The conditional generation results are presented in \cref{table: conditional generation OTD and RMSE comparison}. Across the six datasets, LBDTPP demonstrates the superior overall performance compared with all autoregressive and non-autoregressive TPP baselines. Specifically, our model achieves the best OTD and $\text{RMSE}_{m}$ scores on 5 out of 6 datasets, the best $\text{RMSE}_{\tau}$ on all 6 datasets, and the best sMAPE on 4 out of 6 datasets, while performing comparably on the remaining cases. Its advantage over autoregressive TPP baselines indicates that generating multiple events simultaneously within each block can effectively mitigate error accumulation caused by event-by-event autoregressive generation, leading to more accurate future event predictions. Moreover, LBDTPP outperforms CDiff, the current state-of-the-art diffusion-based TPP baseline, in most cases. This demonstrates that our latent block diffusion approach can more accurately capture the conditional distribution of future sequences given historical observations and produce reliable event forecasts.

\subsection{Sources of Performance Improvement (\textbf{A3})}\label{sec:different block sizes}

In this subsection, we empirically analyze the sources of LBDTPP's performance improvement over baseline methods. The experimental settings are kept consistent with those in \Cref{sec:conditional generation task}, with only the block size $L^{\prime}$ varied among the set $\{1, 2, 4, 8, 16, 20\}$. Through this study, we find that the performance gain of LBDTPP mainly comes from two factors: latent-space diffusion and block-wise generation.

The results of our LBDTPP model under different block sizes are presented in \Cref{fig:blocksize_performance}. We observe that, on most datasets, both OTD and $\text{RMSE}_{m}$ metrics first decrease and then increase as $L^{\prime}$ becomes larger, indicating that the generation quality first improves and then degrades. Moreover, regardless of the block size, LBDTPP generally outperforms the baseline methods reported in \Cref{table: conditional generation OTD and RMSE comparison}. Below, a more detailed analysis explains where the improvement comes from. When $L^{\prime}=1$, LBDTPP reduces to an autoregressive generation paradigm, where events are generated one by one. Even under this setting, LBDTPP still performs better than autoregressive baselines, which demonstrates the benefit of latent-space diffusion. As $L^{\prime}$ increases from 1, LBDTPP generates multiple events in parallel within each block, and its performance improves accordingly. This shows the advantage of block-wise generation, which mitigates the error accumulation issue caused by strictly autoregressive event-by-event sampling.

\begin{figure*}[t]
    \centering
    \includegraphics[width=\linewidth]{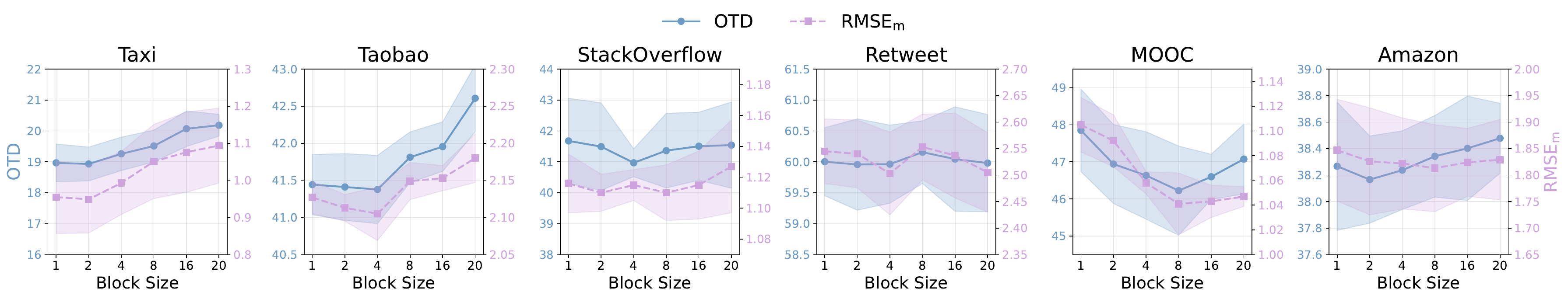}
    \vspace{-15pt}
    \caption{Performance of LBDTPP under different block sizes ($L^{\prime}$) for conditional generation. The left y-axis represents OTD and the right y-axis represents $\text{RMSE}_{m}$. The shaded regions represent the standard deviation.}
    \label{fig:blocksize_performance}
\end{figure*}

\begin{figure}[t] 
    \centering
    \includegraphics[width=0.95\linewidth]{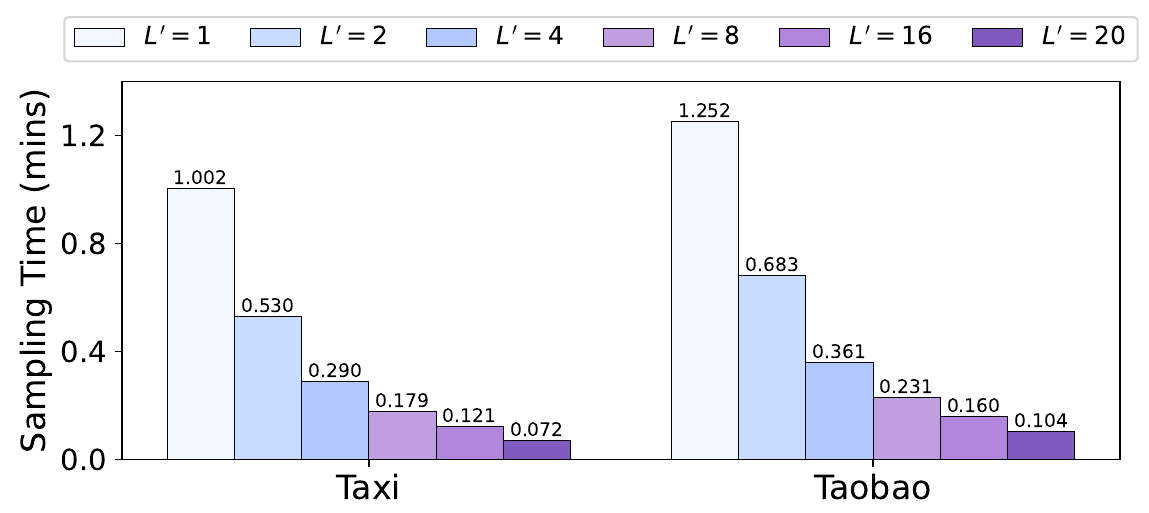}
     \vspace{-2pt}
    \caption{Impact of block size ($L^{\prime}$) on LBDTPP sampling time for conditional generation on Taxi and Taobao datasets.} 
    \label{fig:blocksize_time} 
\end{figure}

This benefit, however, does not keep increasing with larger $L^{\prime}$. Overly large event blocks lead to a coarser factorization of the sequence distribution in \myref{equ:block factorization}, making the conditional distribution of each block harder to learn. They also make the denoising problem more challenging, since each diffusion step needs to jointly recover higher-dimensional event block representations. Nevertheless, when $L^{\prime}=20$, LBDTPP becomes a non-autoregressive model that generates all 20 future events in one shot, and it still outperforms non-autoregressive baselines in most cases. To be specific, compared with TCDDM and CDiff, which perform diffusion directly in the raw event space, the superior performance of LBDTPP further confirms the benefit of performing diffusion in latent space. 

This performance trend as the block size changes is consistent with \Cref{thm:error_accumulation}. Although the theorem is stated for unconditional generation, the same intuition applies to conditional generation after fixing the observed history. Increasing $L^{\prime}$ reduces the number of recursive transitions needed to generate the 20 future events, and thus can reduce prefix-level error accumulation. At the same time, this advantage requires the block-level approximation error and stability to remain well controlled as the block size increases. When $L^{\prime}$ becomes too large, the harder block distribution and denoising problem can increase the local block error, which explains the degradation observed after the optimal block size.

To sum up, these results suggest that the improvement of our LBDTPP model comes from the combination of latent-space diffusion and block-wise generation. We also report the sampling time on the entire test set for different block sizes on the Taxi and Taobao datasets in \Cref{fig:blocksize_time}. The sampling time decreases as the block size increases, which is expected because larger blocks allow the model to generate more events in each sampling round and thus reduce the total number of block sampling rounds. In practice, we recommend using block size $L^{\prime}=4$ when generation quality is the primary concern. If faster generation is preferred, larger block sizes can also be used, since the performance degradation of LBDTPP is relatively mild.


\subsection{Distribution Evaluation (\textbf{A4})}\label{sec:distribution evaluation}

We further evaluate whether LBDTPP captures the empirical distributions of event timestamps and marks in the conditional generation task. As shown in \Cref{fig:distribution_evaluation}, we plot the empirical density of inter-event times and the empirical frequency distribution of event marks for both the ground-truth future events and the events generated by LBDTPP. The results show that the generated inter-event times closely follow the distribution of the true future events, and the generated event marks also match the real mark distribution well. These observations indicate that LBDTPP can effectively capture both temporal and categorical distributional patterns, demonstrating its strong generative forecasting ability beyond standard evaluation metrics.

\begin{figure*}[t]
    \centering
    \includegraphics[width=\linewidth]{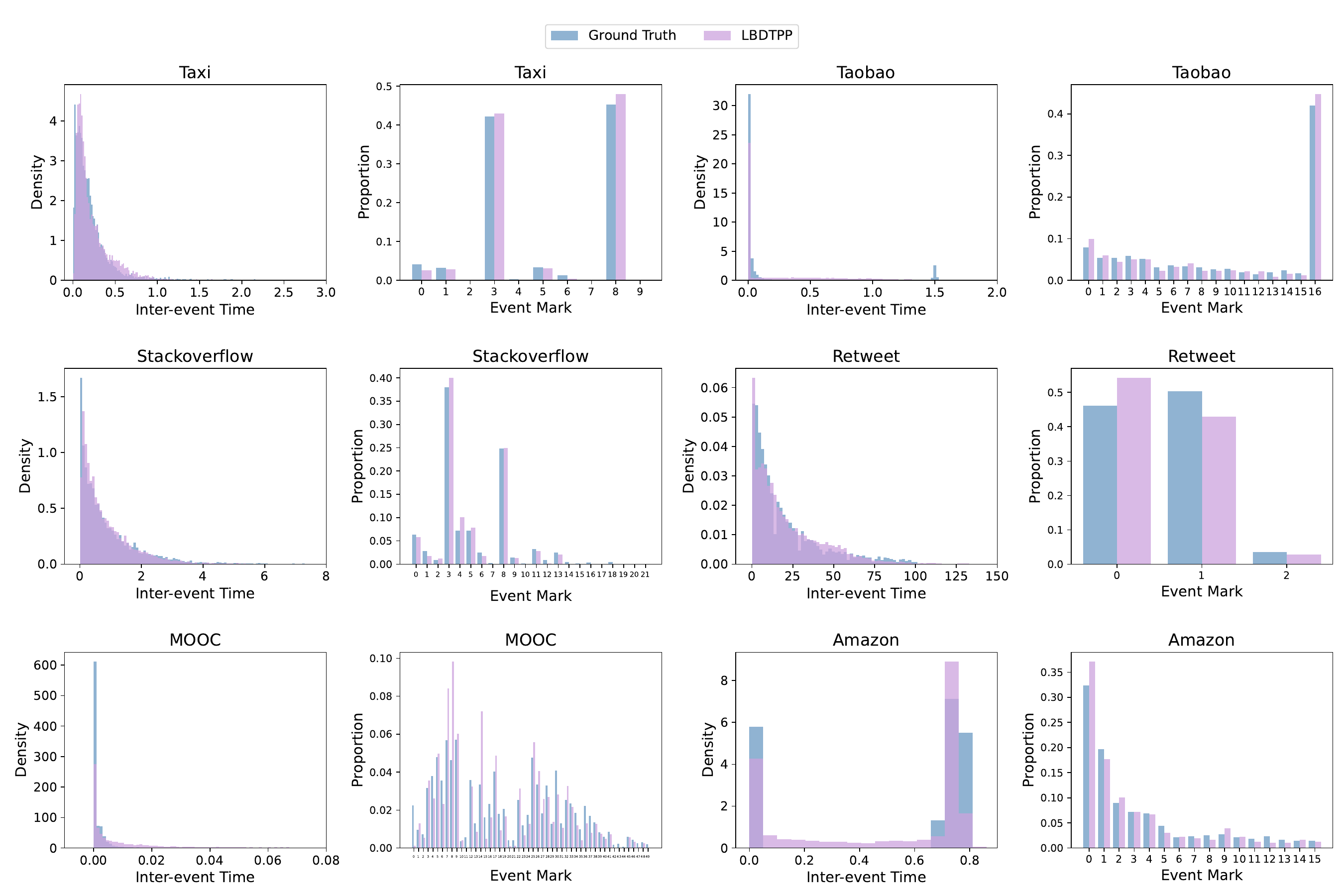}
    \vspace{-15pt}
    \caption{Distribution evaluation of LBDTPP for conditional generation. We plot the empirical density of inter-event times and the empirical frequency distribution of event marks for both the ground-truth future events and the events generated by LBDTPP.}
    \label{fig:distribution_evaluation}
\end{figure*}


\subsection{Hyperparameter Sensitivity (\textbf{A5})}\label{sec:different sampling steps}

We conduct a sensitivity analysis on the number of sampling steps $S$ in the DDIM sampler for conditional generation on Taxi and Taobao datasets. We maintain the same experimental settings as in \Cref{sec:conditional generation task}, altering only $S$ from the set $\{10, 20, 30, 40, 50\}$. The corresponding OTD and $\text{RMSE}_{m}$ trends are illustrated in \Cref{fig:samplingsteps_performance}. To be specific, on the Taxi dataset, LBDTPP achieves OTD scores ranging from 19.039 to 19.258 and $\text{RMSE}_{m}$ scores between 0.967 and 0.993 across different sampling steps. On the Taobao dataset, LBDTPP attains OTD scores from 41.210 to 41.642 and $\text{RMSE}_{m}$ scores between 2.103  and 2.136. From these results, we observe that LBDTPP consistently sustains superior performance compared to the baseline methods in \Cref{table: conditional generation OTD and RMSE comparison}, even with as few as 10 sampling steps. This indicates that LBDTPP can effectively generate high-quality event sequences using a small number of sampling iterations, highlighting its efficiency and enhancing its practicality for real-world applications. Besides, from \Cref{fig:samplingsteps_time}, the sampling time decreases as the number of sampling steps $S$ is reduced, which is expected.

We also study the hyperparameter sensitivity of $\lambda$ in the loss function. We vary $\lambda$ from the set $\{0.01, 0.05, 0.1, 0.5, 1.0\}$ and present the results in \cref{fig:lambda_performance}. We observe that LBDTPP achieves stable performance across different values of $\lambda$ and consistently outperforms all baselines in \Cref{table: conditional generation OTD and RMSE comparison} with respect to OTD and $\text{RMSE}_{m}$ on both datasets. This demonstrates that LBDTPP is robust to the choice of $\lambda$ and can maintain strong performance without requiring extensive hyperparameter tuning.

\subsection{Model Variants (\textbf{A6})}\label{sec:model variants}

We now conduct model variant experiments to examine two design choices of LBDTPP on the conditional generation task using Taxi and Taobao datasets. As described in \Cref{sec:model architecture}, the default LBDTPP encoder contains no learnable parameters: the temporal component is encoded by sinusoidal time embeddings, and the mark component is obtained from a fixed embedding matrix. To evaluate whether a learnable mark representation is beneficial, we consider a variant named LBDTPP-LM, where the embedding matrix $\mathbf{W}$ in \myref{eq:mark-embed} is optimized during training. In addition, the default LBDTPP adopts $\mathbf{z}^{b}$-prediction in the latent block diffusion module, where the denoising network directly predicts the clean latent block. To assess this prediction target, we introduce another variant named LBDTPP-EP, which instead adopts $\boldsymbol{\epsilon}^{b}$-prediction and predicts the Gaussian noise added to the latent block.

\begin{figure}[t]
    \centering
    \includegraphics[width=\linewidth]{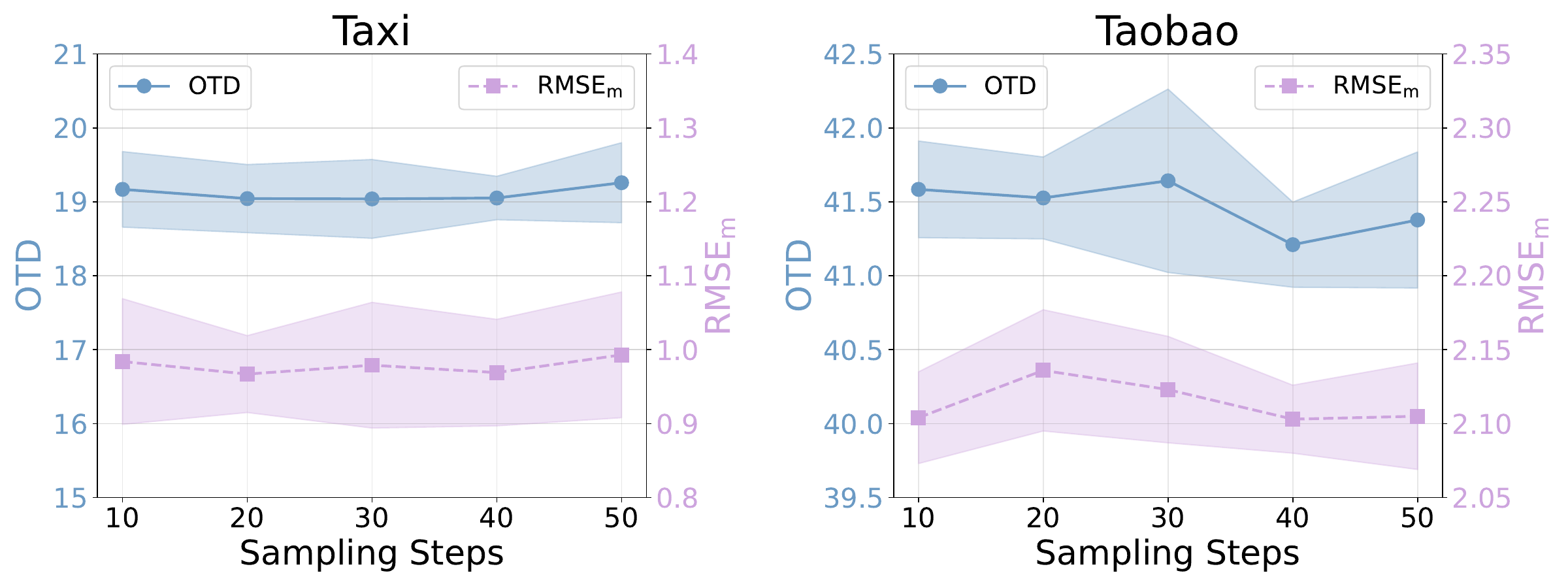}
    \vspace{-15pt}
    \caption{Impact of sampling step ($S$) on LBDTPP performance for conditional generation on Taxi and Taobao datasets.}
    \label{fig:samplingsteps_performance}
\end{figure}

\begin{figure}[t] 
    \centering
    \includegraphics[width=0.96\linewidth]{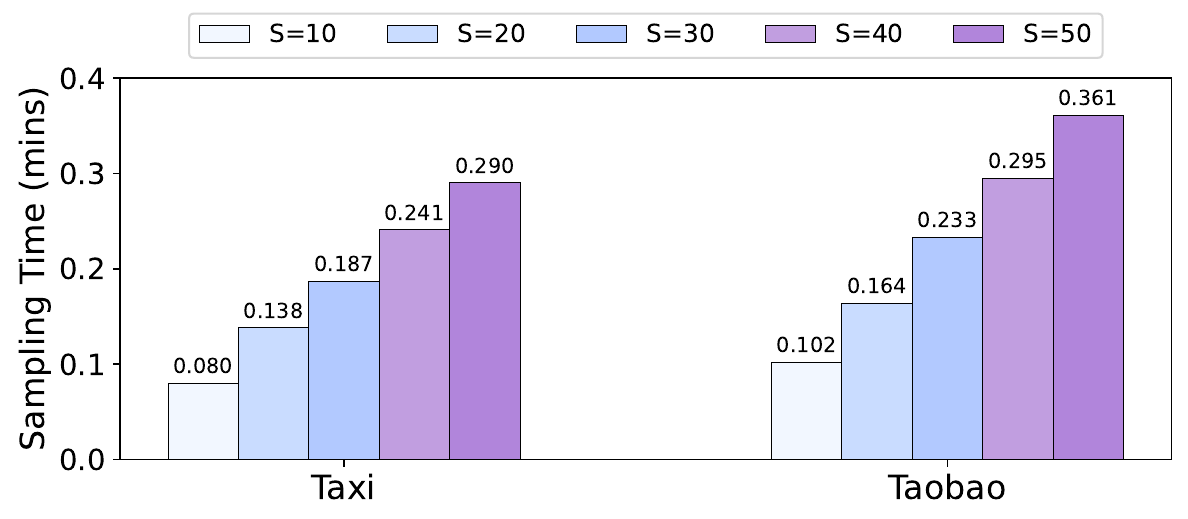}
    \vspace{-2pt} 
    \caption{Impact of sampling step ($S$) on LBDTPP sampling time for conditional generation on Taxi and Taobao datasets.} 
    \label{fig:samplingsteps_time} 
\end{figure}

\begin{figure}[t]
    \centering
    \includegraphics[width=\linewidth]{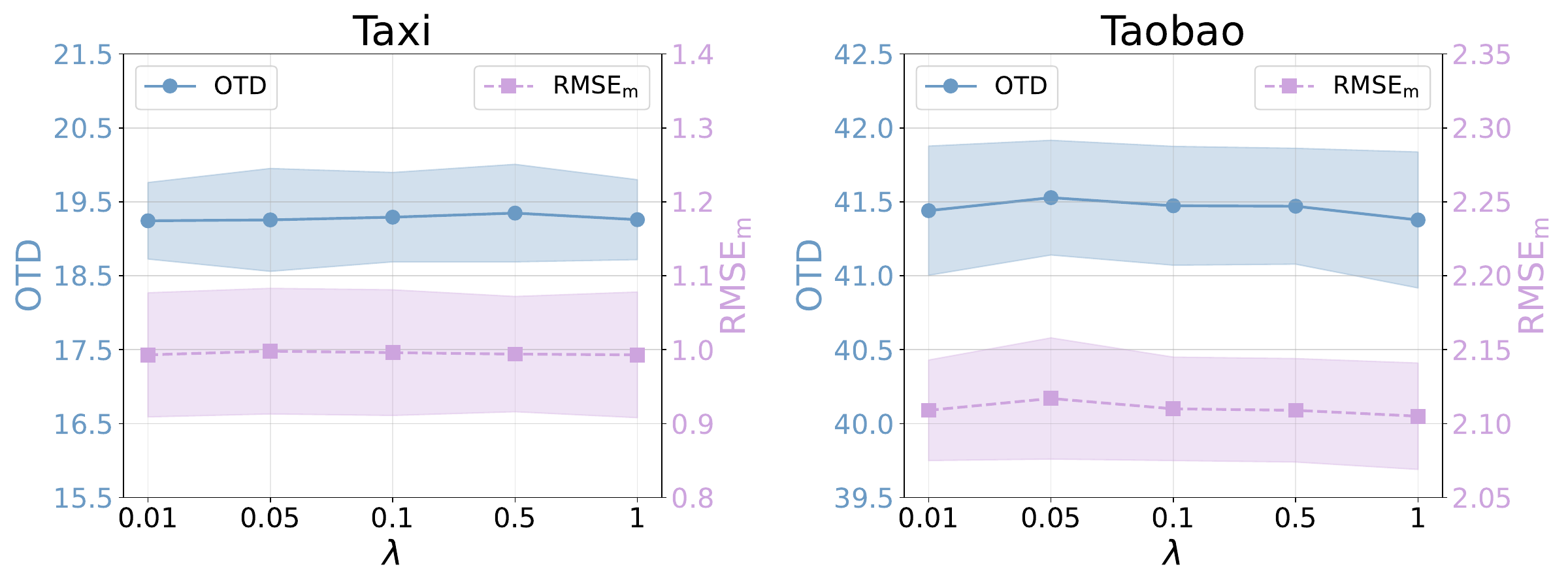}
    \vspace{-14pt}
    \caption{Impact of the values of $\lambda$ on LBDTPP performance for conditional generation on Taxi and Taobao datasets.}
    \label{fig:lambda_performance}
\end{figure}

The results are shown in \Cref{fig:model_variant_performance}. We can observe that LBDTPP slightly outperforms LBDTPP-LM on both datasets, indicating that a parameter-free event encoder is already sufficient for constructing effective latent event representations. This also suggests that introducing a learnable mark embedding matrix does not necessarily improve generation quality under our setting, and the fixed encoder can avoid additional parameters without sacrificing performance. Moreover, LBDTPP consistently performs better than LBDTPP-EP, demonstrating that $\mathbf{z}^{b}$-prediction is more effective than $\boldsymbol{\epsilon}^{b}$-prediction in our latent block diffusion framework. One possible reason is that the clean latent block $\mathbf{z}^{b}$ is exactly the input to the event decoder, so directly predicting $\mathbf{z}^{b}$ provides the decoder with the target latent representation without an additional recovery step. In contrast, $\boldsymbol{\epsilon}^{b}$-prediction first estimates the injected Gaussian noise and then recovers the clean latent block using the relation in \myref{equ:closed form}, which involves dividing by $\sqrt{\bar{\alpha}_k}$. As a result, prediction errors may be amplified when $k$ is large and $\bar{\alpha}_k$ is small, and these errors can subsequently propagate to the event decoder. Since the latent representations jointly encode continuous temporal information and discrete mark information, preserving their clean block structure is particularly important for accurate event sequence generation.

\begin{figure}[t]
    \centering
    \includegraphics[width=\linewidth]{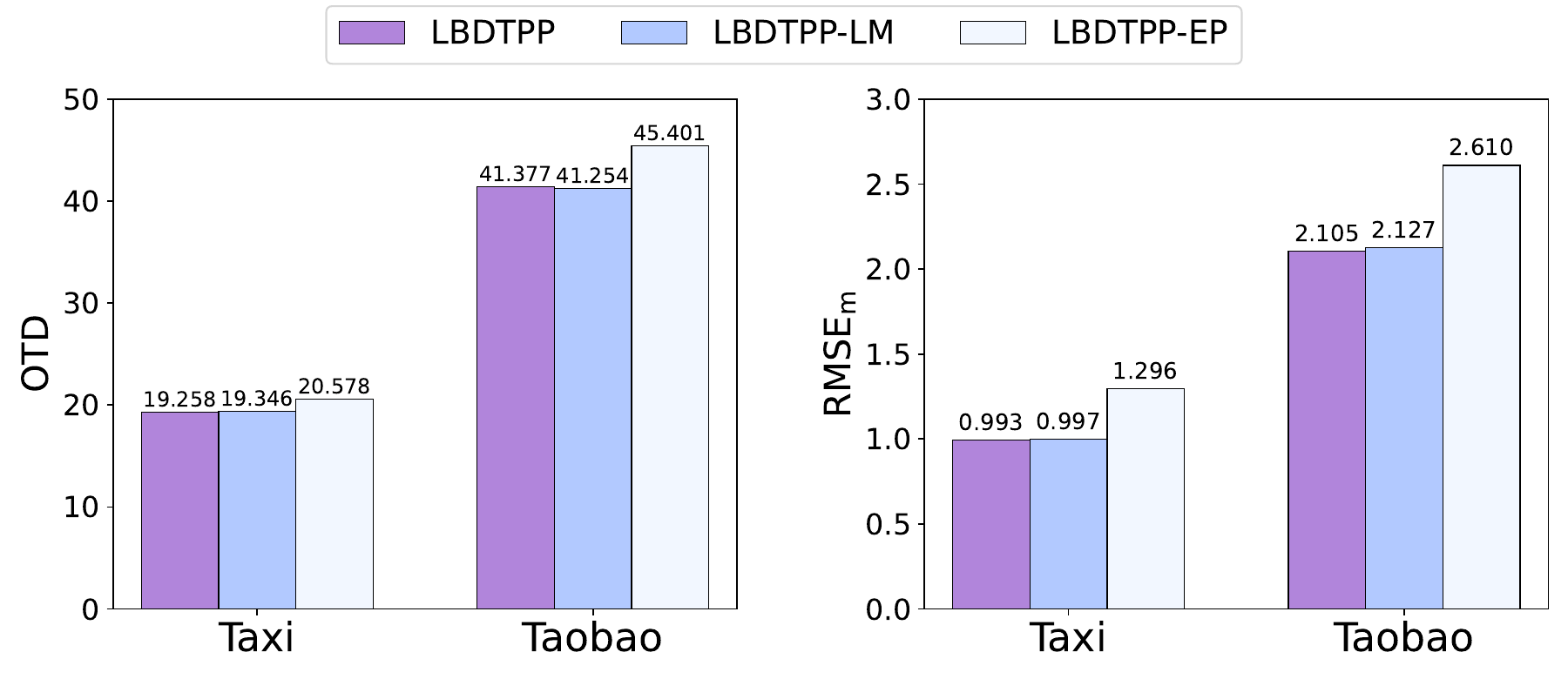}
    \vspace{-15pt}
    \caption{Performance of different LBDTPP variants for conditional generation on Taxi and Taobao datasets.}
    \label{fig:model_variant_performance}
\end{figure}

\begin{figure*}[t]
    \centering
    \includegraphics[width=1\linewidth]{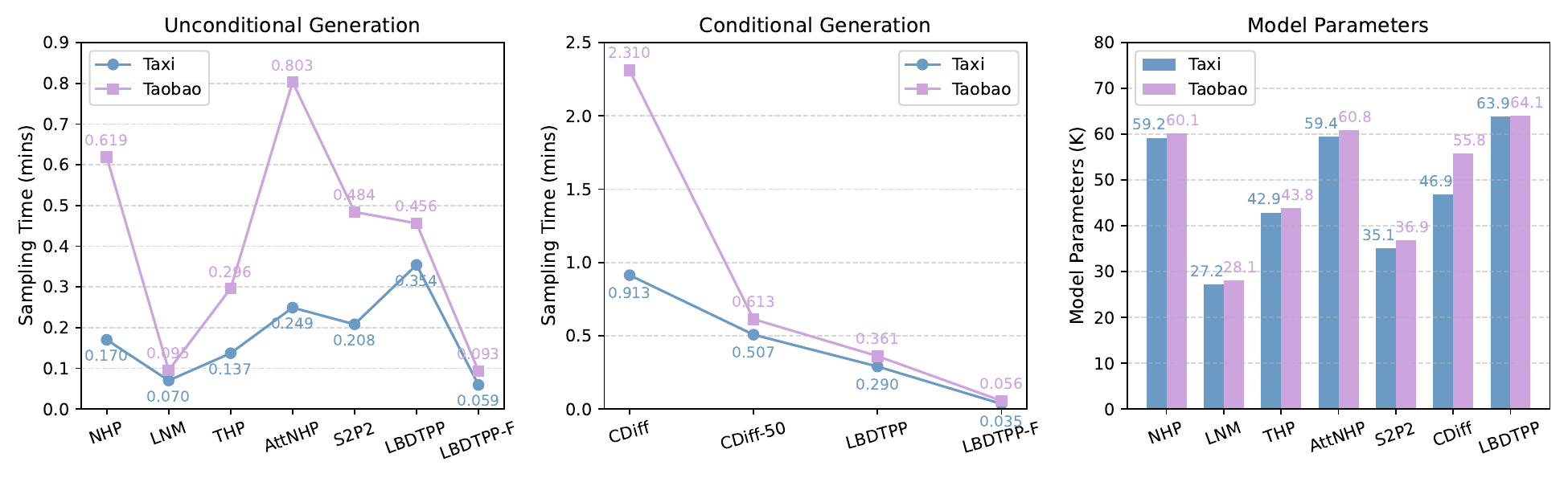}
    \vspace{-15pt}
    \caption{Comparison of sampling time and model parameters with baselines. Sampling time is reported in minutes, and model parameters are reported in thousands (K). Our LBDTPP model achieves comparable or lower sampling time.}
    \label{fig:time_comparison}
\end{figure*}
\subsection{Sampling Time Comparison (\textbf{A7})}\label{sec:sampling time comparison}

We compare the sampling time of our model with that of TPP baselines for both unconditional and conditional generation tasks on Taxi and Taobao datasets. For unconditional generation, we compare LBDTPP with five autoregressive TPP baselines. For conditional generation, we compare LBDTPP with CDiff, the current state-of-the-art diffusion-based TPP baseline. The sampling time is measured on the entire test set and the batch size is set to 32 for all models. Our model keeps the same experimental settings as in Sections~\ref{sec:unconditional generation task} and \ref{sec:conditional generation task}. 

The sampling time results and model parameters are illustrated in \cref{fig:time_comparison}. We can see that our LBDTPP model achieves comparable or lower sampling times in both tasks. LNM models the conditional density function of inter-event times through a mixture of log-normal distributions. A primary advantage of this approach lies in its closed-form sampling expression, which leads to more efficient sampling. The other autoregressive baselines model the conditional intensity function and rely on the thinning algorithm \cite{ogata1981lewis,xue2022hypro} for iterative sampling, resulting in slower speeds. CDiff introduces multiple sampling rounds, where the average of the sampled times and the mode of the sampled marks are used as the final generated events, thereby increasing the overall sampling time. Here, we use 5 rounds from the original code. In contrast, our model requires only one sampling round to achieve high-quality event sequence. For CDiff, we follow the optimal hyperparameter settings reported in the original paper, using 100 sampling steps on Taxi and 200 sampling steps on Taobao, whereas LBDTPP uses 50 sampling steps in all main experiments. For a more controlled comparison, we run CDiff with 50 sampling steps, denoted as CDiff-50, and report its sampling time in \cref{fig:time_comparison}. Even under the same sampling-step budget, our model still shows a clear sampling-time advantage.

Based on the experimental results in Sections~\ref{sec:different block sizes} and \ref{sec:different sampling steps}, LBDTPP's performance is stable with respect to both block size and sampling steps. Note that $L^{\prime}=8$ and $S=50$ are used for unconditional generation, while $L^{\prime}=4$ and $S=50$ are used for conditional generation. We further report the sampling time of a faster version for both tasks by setting $L^{\prime}=20$ and $S=10$, denoted as LBDTPP-F. As shown in \cref{fig:time_comparison}, LBDTPP-F is faster than all baseline models.

\section{Conclusion}\label{sec:conclusion}

We have presented LBDTPP, a novel semi-autoregressive TPP framework that introduces latent block diffusion for modeling asynchronous event sequences. By generating event sequences block by block with parallel generation within each block, LBDTPP supports high-quality, variable-length generation while mitigating error accumulation in autoregressive TPPs and overcoming the fixed-length generation limitation of non-autoregressive diffusion TPPs. Extensive experiments on six real-world datasets demonstrate the superiority of LBDTPP over state-of-the-art TPP baselines in both unconditional and conditional generation tasks. Further analysis confirms the contributions of latent-space diffusion and block-wise generation to the performance improvement. Moreover, LBDTPP achieves comparable or lower sampling times, showcasing its efficiency in generating high-quality event sequences. Future work includes investigating more efficient diffusion sampling techniques to further reduce sampling times, and extending our framework to handle spatio-temporal point processes.

\section{Acknowledgments}
This work was partially supported by the Strategic Priority Research Program of the Chinese Academy of Sciences (No. XDB0680101), the National Natural Science Foundation of China (No. 62472416 and 62402491), and the CAS Project for Young Scientists in Basic Research (No. YSBR-008). The model training was performed on the robotic AI-Scientist platform of Chinese Academy of Sciences.
\bibliographystyle{IEEEtran}
\bibliography{ref}
\newpage
\section*{Supplementary Material}
\subsection{Derivation of NELBO for LBDTPP in Latent Space}\label{sec:nelbo_derivation}

\begin{proof}[Proof of \Cref{prop:blockwise latent elbo}]
Given the latent event sequence representation $\mathbf{z}=(\mathbf{z}^1,\ldots,\mathbf{z}^L)$ partitioned into $B:=L/L^{\prime}$ blocks of length $L^{\prime}$ (with padding applied if $L$ is not divisible by $L^{\prime}$), we denote the index range of the $b$-th block as $\ell_b+1=(b-1)L^{\prime}+1$ to $\ell_{b+1}=bL^{\prime}$, and the $b$-th block as $\mathbf{z}^b=(\mathbf{z}^{\ell_b+1},\ldots,\mathbf{z}^{\ell_{b+1}})$. We use $\mathbf{z}^{<b}=(\mathbf{z}^1,\ldots,\mathbf{z}^{\ell_b})$ to denote all the historical blocks before the $b$-th block. For each $b\in[B]$, we define a forward diffusion process that gradually adds Gaussian noise to the clean block $\mathbf{z}_0^b=\mathbf{z}^b$:
\begin{align*}
q\left(\mathbf{z}_{1: K}^b \mid \mathbf{z}_0^b\right)&=\prod_{k=1}^{K} q\left(\mathbf{z}_k^b \mid \mathbf{z}_{k-1}^b\right),\\
q\left(\mathbf{z}_k^b \mid \mathbf{z}_{k-1}^b\right)&=\mathcal{N}\left(\mathbf{z}_k^b ; \sqrt{\alpha_k} \mathbf{z}_{k-1}^b, (1-\alpha_k) \mathbf{I}\right).
\end{align*}

The reverse denoising process for the $b$-th block starts from $p\left(\mathbf{z}_K^b \mid \mathbf{z}^{<b}\right)=\mathcal{N}\left(\mathbf{z}_K^b ; \mathbf{0}, \mathbf{I}\right)$ and proceeds as follows:
\begin{align*}
p_{\mathbf{\theta}}\left(\mathbf{z}_{0: K}^b\mid \mathbf{z}^{<b}\right)&=p\left(\mathbf{z}_K^b \mid \mathbf{z}^{<b}\right) \prod_{k=1}^K p_{\mathbf{\theta}}\left(\mathbf{z}_{k-1}^b \mid \mathbf{z}_k^b, \mathbf{z}^{<b}\right), \\
p_{\mathbf{\theta}}\left(\mathbf{z}_{k-1}^b \mid \mathbf{z}_{k}^b, \mathbf{z}^{<b}\right)&=\mathcal{N}\left(\mathbf{z}_{k-1}^b; \boldsymbol{\mu}_{\mathbf{\theta}}^b\left(\mathbf{z}_{k}^b, \mathbf{z}^{<b}, k\right), \sigma_k^2\mathbf{I}\right).
\end{align*}

Then the NELBO of our model is obtained as follows:
\allowdisplaybreaks
\begin{align*}
&-\log p_\theta(\mathbf{z})  =-\sum_{b=1}^B \log p_{\mathbf{\theta}}\left(\mathbf{z}^b \mid \mathbf{z}^{<b}\right) \\
& = -\sum_{b=1}^B \log \int p_{\mathbf{\theta}}\left(\mathbf{z}_{0: K}^b\mid \mathbf{z}^{<b}\right) \mathrm{d} \mathbf{z}_{1: K}^b\\
& = -\sum_{b=1}^B \log \int \frac{p_{\mathbf{\theta}}\left(\mathbf{z}_{0: K}^b\mid \mathbf{z}^{<b}\right) q\left(\mathbf{z}_{1: K}^b \mid \mathbf{z}_0^b\right)}{q\left(\mathbf{z}_{1: K}^b \mid \mathbf{z}_0^b\right)} \mathrm{d} \mathbf{z}_{1: K}^b\\
& = -\sum_{b=1}^B \log \mathbb{E}_{q\left(\mathbf{z}_{1: K}^b \mid \mathbf{z}_0^b\right)}\bigg[\frac{p_{\mathbf{\theta}}\left(\mathbf{z}_{0: K}^b\mid \mathbf{z}^{<b}\right)}{q\left(\mathbf{z}_{1: K}^b \mid \mathbf{z}_0^b\right)}\bigg]\\
& \leq -\sum_{b=1}^B \mathbb{E}_{q\left(\mathbf{z}_{1: K}^b \mid \mathbf{z}_0^b\right)} \bigg[\log\frac{p_{\mathbf{\theta}}\left(\mathbf{z}_{0: K}^b\mid \mathbf{z}^{<b}\right)}{q\left(\mathbf{z}_{1: K}^b \mid \mathbf{z}_0^b\right)}\bigg]\\
& = -\sum_{b=1}^B \mathbb{E}_{q\left(\mathbf{z}_{1: K}^b \mid \mathbf{z}_0^b\right)}\Bigg[\log \frac{p\left(\mathbf{z}_K^b | \mathbf{z}^{<b}\right) \prod\limits_{k=1}^K p_{\mathbf{\theta}}\left(\mathbf{z}_{k-1}^b | \mathbf{z}_k^b, \mathbf{z}^{<b}\right)}{\prod\limits_{k=1}^K q\left(\mathbf{z}_k^b \mid \mathbf{z}_{k-1}^b\right)}\Bigg]\\
& = -\sum_{b=1}^B \mathbb{E}_{q\left(\mathbf{z}_{1: K}^b \mid \mathbf{z}_0^b\right)}\Bigg[\log \frac{p\left(\mathbf{z}_K^b \mid \mathbf{z}^{<b}\right) p_{\mathbf{\theta}}\left(\mathbf{z}_{0}^b \mid \mathbf{z}_1^b, \mathbf{z}^{<b}\right)}{q\left(\mathbf{z}_1^b \mid \mathbf{z}_{0}^b\right)} \\
& \qquad\qquad\qquad\qquad\quad+ \log\prod\limits_{k=2}^K\frac{p_{\mathbf{\theta}}\left(\mathbf{z}_{k-1}^b \mid \mathbf{z}_k^b, \mathbf{z}^{<b}\right)}{q\left(\mathbf{z}_k^b \mid \mathbf{z}_{k-1}^b, \mathbf{z}_{0}^b\right)} \Bigg]\\[10pt]
& = -\sum_{b=1}^B \mathbb{E}_{q\left(\mathbf{z}_{1: K}^b \mid \mathbf{z}_0^b\right)}\left[\log \frac{p\left(\mathbf{z}_K^b \mid \mathbf{z}^{<b}\right) p_{\mathbf{\theta}}\left(\mathbf{z}_{0}^b \mid \mathbf{z}_1^b, \mathbf{z}^{<b}\right)}{q\left(\mathbf{z}_1^b \mid \mathbf{z}_{0}^b\right)} \right.\\
& \qquad\qquad\qquad\qquad\quad\left.+ \log\prod\limits_{k=2}^K\frac{p_{\mathbf{\theta}}\left(\mathbf{z}_{k-1}^b \mid \mathbf{z}_k^b, \mathbf{z}^{<b}\right)}{\frac{q\left(\mathbf{z}_{k-1}^b \mid \mathbf{z}_k^b, \mathbf{z}_0^b\right) q\left(\mathbf{z}_k^b \mid \mathbf{z}_0^b\right)}{q\left(\mathbf{z}_{k-1}^b \mid \mathbf{z}_0^b\right)}} \right]\\
& =\sum_{b=1}^B \bigg[\underbrace{-\mathbb{E}_{q\left(\mathbf{z}_1^b \mid \mathbf{z}_0^b\right)}\left[\log p_{\mathbf{\theta}}\left(\mathbf{z}_0^b \mid \mathbf{z}_1^b, \mathbf{z}^{<b}\right)\right]}_{\text {reconstruction term}} \\[10pt]
&\quad\quad\quad\,\,\,+\underbrace{D_{\mathrm{KL}}\left(q\left(\mathbf{z}_K^b \mid \mathbf{z}_0^b\right) \|\, p\left(\mathbf{z}_K^b \mid \mathbf{z}^{<b}\right)\right)}_{\text {prior matching term}}\nonumber\\
&+\underbrace{\sum_{k=2}^K \mathbb{E}_{q\left(\mathbf{z}_k^b \mid \mathbf{z}_0^b\right)}\!\left[D_{\mathrm{KL}}\!\left(q\left(\mathbf{z}_{k-1}^b | \mathbf{z}_k^b, \mathbf{z}_0^b\right) \|\, p_{\mathbf{\theta}}\left(\mathbf{z}_{k-1}^b | \mathbf{z}_k^b, \mathbf{z}^{<b}\right)\right)\right]}_{\text {denoising matching term}}\!\bigg]\\
&=:\sum_{b=1}^B \mathcal{J}_b(\mathbf{z}^b,\mathbf{z}^{<b};\theta).
\end{align*}

In the above derivations, we use Jensen's inequality together with the Markov properties of the forward and reverse processes (Eqs.~\eqref{eq:q_joint} and \eqref{eq:p_joint}), closely following prior work \cite{ho2020denoising,luo2022understanding,arriolablock}. We can further leverage the Gaussian form of the transition distributions (Eqs.~\eqref{equ:forward definition} and \eqref{equ:reverse definition}) to simplify this NELBO into the surrogate latent block diffusion loss function $\mathcal{L}_{\text {LBD}}(\mathbf{z}; \mathbf{\theta})$ in \myref{equ:total block diffusion loss}. Importantly, the Kullback--Leibler (KL) divergence in the denoising matching term admits a closed-form Gaussian expression and can be reduced, up to constants and weighting coefficients, to a mean squared error (MSE)-based denoising loss. We omit the detailed simplification steps here and refer the reader to \cite{luo2022understanding}; in our setting, the same derivation applies after carrying out the simplification blockwise.
\end{proof}

\subsection{Proof of the Generation-Error Accumulation Results}\label{sec:error_accumulation_proof}
We provide the detailed proof of \Cref{thm:error_accumulation} for unconditional generation. The kernels written with ``$\mid$'' below are internal transition kernels from the chain-rule factorization of an unconditional sequence distribution; no external observed sequence is provided. Recall that $d_r(\mathbf{u},\mathbf{v})=\sum_{\ell=1}^{r}\|\mathbf{u}^{\ell}-\mathbf{v}^{\ell}\|_2$, and $W_1^r$ is the Wasserstein-1 distance induced by $d_r$.

\begin{proof}[Proof of \Cref{thm:error_accumulation}]
We first consider event-wise autoregressive generation. Let $\mathsf{P}_{1:\ell}$ be the true joint distribution of the first $\ell$ generated latent events, and let $\mathsf{Q}_{1:\ell}^{\mathrm{AR}}$ be the corresponding joint distribution produced by the event-wise autoregressive generator. Define
\begin{equation}
D_{\ell}=W_1^{\ell}\!\left(\mathsf{P}_{1:\ell},\mathsf{Q}_{1:\ell}^{\mathrm{AR}}\right),\qquad D_0=0.
\end{equation}
Here $D_{\ell}$ measures the accumulated generation error up to event $\ell$ through the distributional discrepancy between the true and generated length-$\ell$ prefixes.
Fix $\ell\geq1$. For any $\eta>0$, choose a coupling $\gamma_{\ell-1}$ of the true and generated length-$(\ell-1)$ event prefixes, namely $\mathsf{P}_{1:\ell-1}$ and $\mathsf{Q}_{1:\ell-1}^{\mathrm{AR}}$, such that
\begin{equation}
\mathbb{E}_{(\mathbf{h},\widehat{\mathbf{h}})\sim\gamma_{\ell-1}}
\left[d_{\ell-1}(\mathbf{h},\widehat{\mathbf{h}})\right]\leq D_{\ell-1}+\eta.
\end{equation}
For each paired event prefix $(\mathbf{h},\widehat{\mathbf{h}})$ with $\mathbf{h},\widehat{\mathbf{h}}\in\mathbb{R}^{(\ell-1)\times D}$, the next true latent event is sampled from the true unconditional transition kernel $\mathsf{P}_{\ell}(\cdot\mid\mathbf{h})$, while the next generated latent event is sampled from the learned autoregressive transition kernel $\mathsf{Q}_{\ell}^{\mathrm{AR}}(\cdot\mid\widehat{\mathbf{h}})$. By the triangle inequality for $W_1^1$ and \Cref{assump:error_accumulation},
\begin{align}
&W_1^1\!\left(\mathsf{P}_{\ell}(\cdot\mid\mathbf{h}),\mathsf{Q}_{\ell}^{\mathrm{AR}}(\cdot\mid\widehat{\mathbf{h}})\right) \nonumber\\
&\quad\leq
W_1^1\!\left(\mathsf{P}_{\ell}(\cdot\mid\mathbf{h}),\mathsf{P}_{\ell}(\cdot\mid\widehat{\mathbf{h}})\right)
+W_1^1\!\left(\mathsf{P}_{\ell}(\cdot\mid\widehat{\mathbf{h}}),\mathsf{Q}_{\ell}^{\mathrm{AR}}(\cdot\mid\widehat{\mathbf{h}})\right) \nonumber\\
&\quad\leq \rho_{\mathrm{AR}}d_{\ell-1}(\mathbf{h},\widehat{\mathbf{h}})+\varepsilon_{\mathrm{AR}}.
\label{eq:supp_next_event_coupling}
\end{align}
For each $(\mathbf{h},\widehat{\mathbf{h}})$, choose an optimal, or arbitrarily close to optimal, coupling of $\mathsf{P}_{\ell}(\cdot\mid\mathbf{h})$ and $\mathsf{Q}_{\ell}^{\mathrm{AR}}(\cdot\mid\widehat{\mathbf{h}})$. Combining it with $\gamma_{\ell-1}$ gives a valid coupling of $\mathsf{P}_{1:\ell}$ and $\mathsf{Q}_{1:\ell}^{\mathrm{AR}}$. Under this coupling, the expected length-$\ell$ sequence distance is bounded by
\begin{align}
D_{\ell}
&\leq \mathbb{E}\!\left[d_{\ell-1}(\mathbf{h},\widehat{\mathbf{h}})\right]
+\mathbb{E}\!\left[W_1^1\!\left(\mathsf{P}_{\ell}(\cdot\mid\mathbf{h}),\mathsf{Q}_{\ell}^{\mathrm{AR}}(\cdot\mid\widehat{\mathbf{h}})\right)\right] \nonumber\\
&\leq (1+\rho_{\mathrm{AR}})\mathbb{E}\!\left[d_{\ell-1}(\mathbf{h},\widehat{\mathbf{h}})\right]+\varepsilon_{\mathrm{AR}} \nonumber\\
&\leq (1+\rho_{\mathrm{AR}})(D_{\ell-1}+\eta)+\varepsilon_{\mathrm{AR}}.
\end{align}
Letting $\eta\downarrow0$ yields the recurrence
\begin{equation}
D_{\ell}\leq (1+\rho_{\mathrm{AR}})D_{\ell-1}+\varepsilon_{\mathrm{AR}}.
\end{equation}
Unrolling it from $D_0=0$ gives
\begin{equation}
D_L\leq \varepsilon_{\mathrm{AR}}\sum_{j=0}^{L-1}(1+\rho_{\mathrm{AR}})^j
=\varepsilon_{\mathrm{AR}}A_L(\rho_{\mathrm{AR}}),
\end{equation}
which proves \myref{eq:ar_error_bound}.

We next prove the block-wise result for unconditional generation. Let $\mathsf{P}_{1:b}^{\mathrm{BL}}$ and $\mathsf{Q}_{1:b}^{\mathrm{BL}}$ denote the true and generated joint distributions of the first $b$ latent blocks, equivalently the first $bL^{\prime}$ latent events. Define
\begin{equation}
E_b=W_1^{bL^{\prime}}\!\left(\mathsf{P}_{1:b}^{\mathrm{BL}},\mathsf{Q}_{1:b}^{\mathrm{BL}}\right),\qquad E_0=0.
\end{equation}
Here $E_b$ quantifies the corresponding generation error after $b$ generated blocks.
Repeating the preceding argument at the block level, for any coupling of previous block prefixes $(\mathbf{g},\mathbf{g}')$ with $\mathbf{g},\mathbf{g}'\in\mathbb{R}^{(b-1)L^{\prime}\times D}$, the triangle inequality and \Cref{assump:error_accumulation} give
\begin{align}
&W_1^{L^{\prime}}\!\left(\mathsf{P}_{b}^{\mathrm{BL}}(\cdot\mid\mathbf{g}),\mathsf{Q}_{b}^{\mathrm{BL}}(\cdot\mid\mathbf{g}')\right) \nonumber\\
&\quad\leq
W_1^{L^{\prime}}\!\left(\mathsf{P}_{b}^{\mathrm{BL}}(\cdot\mid\mathbf{g}),
\mathsf{P}_{b}^{\mathrm{BL}}(\cdot\mid\mathbf{g}')\right) \nonumber\\
&\qquad+
W_1^{L^{\prime}}\!\left(\mathsf{P}_{b}^{\mathrm{BL}}(\cdot\mid\mathbf{g}'),
\mathsf{Q}_{b}^{\mathrm{BL}}(\cdot\mid\mathbf{g}')\right) \nonumber\\
&\quad\leq \rho_{\mathrm{BL}}d_{(b-1)L^{\prime}}(\mathbf{g},\mathbf{g}')+\varepsilon_{\mathrm{BL}}.
\end{align}
Since the sequence metric is additive across blocks, this yields
\begin{equation}
E_b\leq (1+\rho_{\mathrm{BL}})E_{b-1}+\varepsilon_{\mathrm{BL}}.
\end{equation}
Unrolling the recurrence from $E_0=0$ gives
\begin{equation}
E_B\leq \varepsilon_{\mathrm{BL}}\sum_{j=0}^{B-1}(1+\rho_{\mathrm{BL}})^j
=\varepsilon_{\mathrm{BL}}A_B(\rho_{\mathrm{BL}}),
\end{equation}
which proves \myref{eq:block_error_bound}.

It remains to prove \myref{eq:accumulation_factor}. Since $A_n(\rho)=\sum_{j=0}^{n-1}(1+\rho)^j$ is nondecreasing in $\rho$, the conditions $\varepsilon_{\mathrm{BL}}\leq L^{\prime}\varepsilon_{\mathrm{AR}}$ and $\rho_{\mathrm{BL}}\leq \rho_{\mathrm{AR}}=\rho$ imply that the block-wise upper bound is at most $L^{\prime}\varepsilon_{\mathrm{AR}}A_B(\rho)$, whereas the event-wise upper bound is $\varepsilon_{\mathrm{AR}}A_L(\rho)$. If $\rho=0$, then $A_B(0)=B$ and $A_L(0)=L=BL^{\prime}$, so $L^{\prime}A_B(0)/A_L(0)=1$. If $\rho>0$, set $a=1+\rho>1$. Since $L=BL^{\prime}$,
\begin{equation}
\frac{A_L(\rho)}{A_B(\rho)}
=\frac{a^{BL^{\prime}}-1}{a^B-1}
=1+a^B+a^{2B}+\cdots+a^{(L^{\prime}-1)B}
\geq L^{\prime}.
\end{equation}
Therefore, $L^{\prime}A_B(\rho)/A_L(\rho)\leq1$, with strict inequality whenever $\rho>0$ and $L^{\prime}>1$. This completes the proof.
\end{proof}

\begin{proof}[Proof of \myref{eq:decoder_transfer}]
Let $\mu$ and $\nu$ be two latent sequence distributions on $\mathbb{R}^{L\times D}$, and let $\pi$ be any coupling of $\mu$ and $\nu$. If $(\mathbf{U},\mathbf{V})\sim\pi$, then $\big(g_{\phi}(\mathbf{U}),g_{\phi}(\mathbf{V})\big)$ is a coupling of the push-forward distributions $(g_{\phi})_{\#}\mu$ and $(g_{\phi})_{\#}\nu$. If $g_{\phi}$ is $L_{\mathrm{dec}}$-Lipschitz,
\begin{equation}
\mathbb{E}_{\pi}\!\left[\Delta_L\!\left(g_{\phi}(\mathbf{U}),g_{\phi}(\mathbf{V})\right)\right]
\leq L_{\mathrm{dec}}\mathbb{E}_{\pi}\!\left[d_L(\mathbf{U},\mathbf{V})\right].
\end{equation}
Taking the infimum over all couplings $\pi$ of $\mu$ and $\nu$ gives
\begin{equation}
W_{\Delta_L}\!\left((g_{\phi})_{\#}\mu,(g_{\phi})_{\#}\nu\right)
\leq L_{\mathrm{dec}}W_1^L(\mu,\nu).
\end{equation}
Substituting $\mu=\mathsf{P}_{1:L}$ and $\nu=\mathsf{Q}_{1:L}$ proves \myref{eq:decoder_transfer}.
\end{proof}

\end{document}